\documentclass{article}
\RequirePackage{fix-cm}
\DeclareMathSizes{10}{10}{8}{5}

\usepackage{a4wide}

\usepackage[utf8]{inputenc} 
\usepackage[T1]{fontenc}
\usepackage{lmodern}
\usepackage{dsfont}
\usepackage{booktabs}
\usepackage{url}
\usepackage{hyperref}
\usepackage{multicol}
\usepackage{todonotes}
\usepackage[ruled, linesnumberedhidden, vlined, commentsnumbered]{algorithm2e}
\usepackage{tikz}
\usetikzlibrary{arrows, fit, shapes, automata}
\usepackage{csquotes}
\usepackage[nottoc,numbib]{tocbibind}
\usepackage{subcaption}
\usepackage{graphicx}  
\usepackage{bm}
\usepackage{amsmath}
\usepackage{amsthm}
\usepackage{authblk}
\usepackage{pgfplots}
\usepackage{tabularx}
\usepackage{rotating}
\usepackage{paralist}
\usepackage{amssymb}
\usepackage{dsfont}
\usepackage{multirow}
\usepackage{tabu}

\newtheorem{example}{Example}
\newtheorem{definition}{Definition}
\newtheorem{proposition}{Proposition}
\newtheorem{notation}{Notation}

\newcommand{\prob}{P}
\newcommand{\lab}{L}
\newcommand{\dialogue}{D}
\newcommand{\graph}{{\cal G}}
\newcommand{\arcs}{{\sf Arcs}}
\newcommand{\nodes}{{\sf Nodes}}
\newcommand{\options}{{\sf Options}}
\newcommand{\positmoves}{{\sf PositMoves}}
\newcommand{\menumoves}{{\sf MenuMoves}}
\newcommand{\menupick}{{\sf MenuPick}}
\newcommand{\menulisting}{{\sf MenuListing}}
\newcommand{\dsystem}{{\sf System}}
\newcommand{\duser}{{\sf User}}
\newcommand{\dusertag}{{\sf U}}
\newcommand{\reward}{{\sf Reward}}
\newcommand{\concernscore}{{\sf ConcernScore}}
\newcommand{\prefscore}{{\sf PrefScore}}
\newcommand{\conpref}{\preceq}
\newcommand{\nonchosenscore}{{\sf NonChosenScore}}

\newcommand{\args}{{\sf Args}}
\newcommand{\attackers}{{\sf Attackers}}
\newcommand{\nrej}{{\sf null}^{\sf rej}}
\newcommand{\nacc}{{\sf null}^{\sf acc}}
\newcommand{\concerns}{{\sf Con}}
\newcommand{\concerndomain}{{\cal C}}

\newcommand{\siblings}{{\sf Siblings}}
\newcommand{\exsibcon}{{\sf ExSibCon}}
\newcommand{\sibcon}{{\sf SibCon}}
\newcommand{\initarg}{{\sf Initial}}

\title{Strategic Argumentation Dialogues for Persuasion: Framework and Experiments Based\\ on Modelling the Beliefs and Concerns of the Persuadee}

\date{}
\author[1,3]{Emmanuel Hadoux}
\author[1]{Anthony Hunter}
\affil[1]{Department of Computer Science, University College London, United Kingdom} 
\author[2]{Sylwia Polberg} 
\affil[2]{School of Computer Science and Informatics, Cardiff University, United Kingdom}
\affil[3]{Scribe Labs, London, United Kingdom}

\begin{document}

\maketitle


\begin{abstract}

Persuasion is an important and yet complex aspect of human intelligence. 
When undertaken through dialogue, the deployment of good arguments, and therefore counterarguments, clearly has a significant effect on the ability to be successful in persuasion. 
Two key dimensions for determining whether an argument is \enquote{good} in a particular dialogue
are the degree to which the intended audience believes the argument and counterarguments, 
and the impact that the argument has on the concerns of the intended audience.
In this paper, we present a framework for modelling persuadees in terms of their beliefs and concerns, and for harnessing these models in optimizing the choice of move in persuasion dialogues. 
Our approach is based on the Monte Carlo Tree Search which allows optimization in real-time.
We provide empirical results of a study with human participants showing that our automated persuasion system based on this technology is superior to a baseline system that does not take the beliefs and concerns into account in its strategy. 

\end{abstract}

\newpage
\tableofcontents
\newpage

\section{Introduction} 

Persuasion is an important and multifaceted human facility. The ability to induce another party to believe or do something 
is as essential in commerce and politics as it is in many aspects of daily life. 
We can consider examples such as a doctor trying to get a patient to enter a smoking cessation programme,  
a politician trying to convince people to vote for him in the elections, or even just a child asking a parent for a rise in pocket money.
There are many components that boost the effectiveness of persuasion, and simple things such as how someone is dressed 
or a compliment can affect the person they are trying to convince. Nevertheless, arguments are a crucial part of persuasion, 
and resolving a given person's doubts and criticisms is necessary to win them over.  

While arguments can be implicit, as in a product advert, or explicit, as in a discussion with a doctor, in both cases they need to be selected
with the target audience in mind. For instance, a doctor encouraging a patient who is a parent to take better care of their heart can use 
references to the patient's family in order to succeed. It is therefore valuable for the persuader to have an understanding of the persuadee and of what might work better with them. This is particularly challenging in situations where the persuadees take on an active rather 
than passive role and can voice and change their opinions throughout the persuasion attempt. In this paper, we focus on 
the following two dimensions in which a potential persuadee may judge arguments in the context of a dialogue. 

\begin{description}

\item[Beliefs]  Arguments are formed from premises and a claim, either of which may be explicit or partially implicit. An agent can express a belief in an argument based on the agent's belief in the premises being true, the claim being implied by the premises, and the claim being true. 
There is substantial evidence in the behaviour change literature that shows the importance of the beliefs of a persuadee in affecting the likelihood that the persuasion attempt is successful (see for example the review by Ogden \cite{ogden2012health}). Furthermore, beliefs 
can be used as a proxy for fine-grained argument acceptability, the need for which was highlighted by empirical studies conducted in \cite{Rahwan2011,PolbergHunter2018ijar}. 

\item[Concerns]  Arguments are statements that contain information about the agent and/or the world. Furthermore, they can refer to impacts on the agent and/or the world, which in turn may relate to the concerns of the agent. In other words, some arguments may have a significant impact on what the agent is concerned about. In empirical studies, it has been shown that taking the persuadee's concerns into account can improve the likelihood that persuasion is successful \cite{HadouxHunter2018submission,Chalaguine19,Chalaguine2020}.

\end{description}
 
To illustrate how beliefs (respectively concerns) arise in argumentation, and how they can be harnessed for more effective persuasion, consider Example \ref{ex:belief} (respectively Example \ref{ex:concern}).

\begin{example}
\label{ex:belief}
Consider a health advisor who wants to persuade a student to join a smoking cessation programme (\emph{i.e.}, a health programme designed to help someone give up smoking). The student may be expressing reluctance to join but not explaining why. Through experience, the advisor might guess that the student believes one of the following arguments.
\begin{itemize}
\item Argument 1: If I give up smoking, I will get more anxious about my studies, I will eat less, and I will lose too much weight.
\item Argument 2: If I give up smoking, I will start to eat more as a displacement activity while I study, and I will get anxious as I will put on too much weight.  
\end{itemize}
Based on the conversation so far, the health advisor has to judge whether the student believes Argument 1 or Argument 2. 
With that prediction, the advisor can try to present an appropriate argument to counter the student's belief in the argument, and thereby overcome the student's barrier to joining the smoking cessation programme. For instance, if the advisor thinks it is Argument 1, they 
can suggest that as part of the smoking cessation programme, the student can join free yoga classes to overcome any stress that they might feel from the nicotine withdrawal symptoms.  
\end{example}






\begin{example}
\label{ex:concern}
Consider a volunteer street-fundraising for a hospital charity who has managed to engage in a conversation with a passerby. 
\begin{itemize}
\item Argument 1: Supporting this hospital will fund innovative cancer research.
\item Argument 2: Supporting this hospital will fund specialized hearing equipment for deaf people. 
\end{itemize}

The volunteer is fundraising in a university area and managed to stop a person that is likely to be a professor in the nearby institution. Hence, they may guess that supporting research is something very close to that person's interests and concerns, probably closer than funding hearing equipment. The volunteer is likely to have just one chance at convincing the person to sign up, and will regard Argument 1 as the more convincing argument to present. 
\end{example}

So in Example \ref{ex:belief}, the student has the same concerns, but different beliefs, associated with the arguments. In contrast, in Example \ref{ex:concern},  the passerby has the same beliefs, but different concerns, associated with the arguments. 
We therefore see both the concerns and beliefs as being orthogonal kinds of information that an agent might have about an argument, 
and knowing about them can be valuable to a persuader. 

The importance of such information becomes even more apparent when the persuader is an artificial, and not a human agent, 
and thus requires a formal representation of the persuadee. In recent years, there has been an increase in demand 
for such agents in various areas, from e-commerce to e-health. Artificial agents can offer accessibility or availability rarely possible 
for human agents, or be deployed in scenarios in which hiring appropriate specialists can be difficult due to the numbers required. 
Furthermore, using such tools may also have unexpected benefits. For instance, research shows that artificial agents can help 
in reducing barriers to healthcare and that the patients may be more truthful with them in contrast to a human specialist \cite{Gratch2014}.

In this work we therefore focus on developing automated persuasion systems (APSs). An APS
plays the role of the persuader and engages in a dialogue with a \emph{user} (the persuadee)
in order to convince them to accept 
a certain persuasion goal (\emph{i.e.}, the argument that encapsulates the reason for a change of behaviour in some respect) 
\cite{Hunter2016comma}. 
Whether an argument is convincing or not depends on the context of the dialogue and on the characteristics of the persuadee. 
Thus, an APS may maintain a model of the persuadee, use it to predict what arguments they may know about 
and/or believe, and harness this information in order to improve the choices of move in a dialogue. 

Beliefs and concerns are of course not the only dimensions one can consider. For instance, the emotional response of an agent to an 
argument is also worth investigating (for example \cite{HadouxHunter2018foiks}) as part of ongoing research on the development of advanced APSs. However, our aim here is to bring together the research on beliefs and concerns, which so far have been considered 
separately, in order to produce a next-level APS that could then be considered as a base for further improvements.   

The majority of our previous research has focused on beliefs in arguments as being a key aspect of a user model for making good choices of moves in a dialogue. 
To represent and reason with beliefs in arguments, we have used the epistemic approach to probabilistic argumentation \cite{Thimm12,Hunter2013ijar,BaroniGiacominVicig14,HunterThimm2016ijar,PolbergHunterThimm17}, the value of which has been 
supported by experiments with participants \cite{PolbergHunter2018ijar}. 
In applying this approach to modelling a persuadee's beliefs in arguments, we have developed methods for:
\begin{inparaenum}[(1)]
\item updating beliefs during a dialogue \cite{Hunter2015ijcai,Hunter2016sum,HunterPotyka2017};
\item efficiently representing and reasoning with the probabilistic user model \cite{HadouxHunter16}; 
\item modelling uncertainty in the modelling of 
persuadee beliefs \cite{Hunter2016ecai,HadouxHunter2018aamas}; 
\item harnessing decision rules for optimizing the choice of argument based on the user model \cite{HadouxHunter17,HadouxHunter2018foiks}; 
\item crowdsourcing the acquisition of user models 
based on beliefs \cite{HunterPolberg17ictai};
\item modelling a domain in a way that supports 
the use of the epistemic approach \cite{Hunter2018domain}.
\end{inparaenum}
These developments for taking belief into account offer a well-understood theoretical and computationally viable framework for applications such as behaviour change. 

However, belief in an argument is not the only dimension of a user model that could 
be taken into account. Recent research provides some evidence that taking concerns into account can improve the persuasiveness of a dialogue \cite{HadouxHunter2018submission,Chalaguine19}. 
Thus, in order to model users better, it is worth exploring a combination of consideration of a user's beliefs with consideration of their concerns in a coherent framework for strategic argumentation. This however creates the need to empirically evaluate with participants the combined use of beliefs and concerns in persuasion.


The aim of this paper is therefore to provide a computational approach to strategic argumentation for persuasion that takes both the concerns and the beliefs of the persuadee into account. They will be used to provide a more advanced user model
which can be harnessed by a decision-theoretic APS to choose the arguments to present in a dialogue. To render this approach viable for real-time applications and to dynamically update an APS's strategy as the dialogue progresses, we present an approach based on Monte Carlo Tree Search. We evaluate our proposal in an empirical study with human participants
using an APS based on this technology which we will refer to as the {\em advanced system}. We compare its performance with an approach that does not rely either on beliefs nor concerns, which we will refer to as the {\em baseline system}. 

We proceed as follows:
(Section \ref{section:argumentation}) We review basic concepts about computational persuasion;
(Sections \ref{section:domainmodelling}, \ref{section:usermodelling} and \ref{sec:dialogues}) We present our setting, from domain and user modelling to dialogue protocol; 
(Section \ref{sec:strategies}) We present our framework for optimizing choices of moves in persuasion dialogues; 
(Section \ref{section:datascience}) We present our approach to acquiring and harnessing the crowdsourced data for user models; 
(Section \ref{section:experiments}) We present our experiments for evaluating our technology in automated persuasion systems; 
(Section \ref{section:literaturereview}) We discuss our work with respect to the related literature;  
and
(Section \ref{section:discussion}) We discuss our contributions and future work.


\section{Framework for computational persuasion}
\label{section:argumentation}
  
In order to work, an automated persuasion system needs a formal model of certain important components. This includes 
a formalization of domain and user models and a dialogue engine, as depicted in Figure \ref{fig:persuasionframework}.

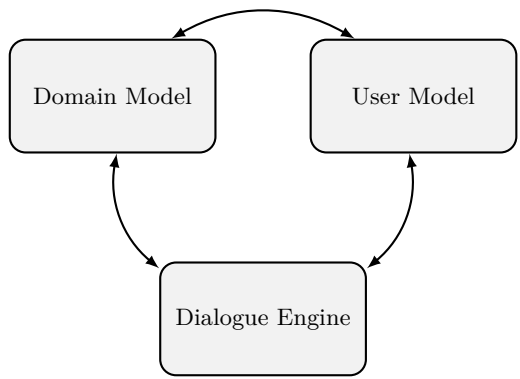
\begin{figure}[ht]
\centering
\begin{tikzpicture}[->,>=latex,thick]
\tikzstyle{every node}=[{draw,text centered, shape=rectangle, rounded corners=6pt, fill=gray!10,font=\small,text width = 25mm, minimum height=15mm}]
			\node (a1)  {Domain Model};
			\node (a2) at ([shift=({0:4 cm})]a1) {User Model};
			\node (a3) at ([shift=({-60:4 cm})]a1) [yshift = 0.5cm]  {Dialogue Engine};
			\path	([xshift=-6mm] a1.north east)[<->,bend left] edge ([xshift=6mm] a2.north west);
			\path	(a2)[<->, bend left] edge  (a3); 
			\path	(a3)[<->, bend left] edge (a1);
\end{tikzpicture} 
\caption{\label{fig:persuasionframework}
In an APS, the dialogue engine is responsible for generating the system moves and for accepting the user's moves. To do this, the dialogue engine incorporates the dialogue protocol, and manages the strategic choice of move for the APS, which it does by consulting the user model and the domain model. The domain model and user model share information about the arguments and the concerns associated with them.
}
\end{figure}

 
The main purpose of the \textbf{domain model} is to represent the arguments that can be uttered in the dialogue
by the system as well as the arguments that the user may entertain. Generally, it can be seen as a source of knowledge 
for the APS, and has to have sufficient breadth and depth in order to tackle any issues that may be raised during 
a persuasion attempt. It can be populated manually \cite{HadouxHunter2018submission,Reisert2019}, through argument mining \cite{Lippi2016,Stede2019}, through argument harvesting/crowdsourcing \cite{CHHP2018,Chalaguine19}, 
or by any combination of these methods. The contents of an argument can be textual, as seen in 
natural language dialogues or newspapers, or follow a particular logical structure, as obtained from formal knowledge bases. 
They can also be annotated according to their nature, \emph{e.g.} whether they are factual, represent social norms, 
user commitments, etc \cite{Hunter2018domain}. Independently of the choice of domain, 
it is important to highlight that it has an impact both on the user model and the dialogue engine as the source of 
kinds of arguments that can be expressed. We present our domain model in Section \ref{section:domainmodelling}, and discuss how it is populated in Section \ref{section:datascience}.

The \textbf{user model} is meant to represent what the APS believes to be true about the user. The information contained 
in the model can be harnessed during a dialogue in order to tailor it to the user and to boost its effectiveness. 
We can consider various kinds of information, such as beliefs, intentions, desires, emotions, and more. 
In this paper, we will focus on beliefs and concerns, and discuss our user models in Section \ref{section:usermodelling}. Details on how a user model is populated will be given in Section \ref{section:datascience}.
 
The main purpose of the \textbf{dialogue engine} is to offer a formal representation of a dialogue and its rules, and to 
decide on what arguments and when the system should utter during a discussion. 
In terms of representation, in this paper we assume that a \textbf{dialogue} is a sequence of moves and that 
the participants take turns in expressing themselves. 
Generally, in formal models of argument, there is a wide variety of types of move that we could consider in our dialogues, 
including positing an argument, querying, conceding, and more. They might not always be valid at any point 
in the dialogue - for instance, it is common to expect that a query move is matched by an appropriate response, 
or that a termination move should not be followed by any other moves, and so on. 
Consequently, every dialogue has to follow a particular \textbf{protocol}, which specifies what are the allowed moves that can be made at each step of a dialogue. The protocol is public and every participant has to adhere to it. 

There are many possible protocols that we could define even with a limited range of moves \cite{Prakken2005,CaminadaPodlaszewki12,BH09,FanToni11,BenchCapon2002}. Often they can be classified 
as symmetric or asymmetric, \emph{i.e.} whether all participants in the dialogue can perform all kinds of moves and 
express themselves freely, or if some restrictions are imposed on some. The latter are quite common for APSs, 
where the system may present a selection of moves to choose from to the user rather than ask them for natural language input. 
This removes the need for NLP capabilities in an APS. In Section \ref{sec:dialogues} we will provide the description of the dialogue moves and protocol that we have used in our experiments.

During a dialogue under a given protocol, every participant has to choose how to respond to the moves of other participants. 
The way they select their dialogue moves is called a \textbf{strategy} \cite{Thimm14}. This includes random approaches
(\emph{i.e.} an arbitrary move out of the permitted ones is selected) in which the participant does not care about their performance, 
as well as those in which some evaluation is taking place in order to ensure that the likelihood of the desired outcome increases. 
The random approach is typically used to serve as a baseline against which more advanced methods can be tested.
In this paper, we will present a strategy module for our APS in Section \ref{sec:strategies}, with the aim of improving the likelihood that the APS is persuasive, and present an evaluation of our strategic APS against a baseline system in Section \ref{section:experiments}. 

\section{Domain modelling}
\label{section:domainmodelling} 

In the context of this paper, we focus on arguments as they can be found in newspaper articles or in discussions between 
humans, \emph{e.g.} on forums or social media. In other words, we assume that they are pieces of text (1-2 sentences), 
and are either a short claim representing a persuasion goal or fact (\emph{e.g.} \enquote{\emph{Universities should continue charging students the \pounds 9K fee.}}), or loosely follow the premise-claim construction (\emph{e.g.} \enquote{\emph{Students should regard their university education as an investment in their future, and so they should agree to pay the student fees.}}). In order to 
improve readability, the claim may be left implicit if it is sufficiently obvious. Every argument will also be associated 
with an appropriate tag for ease of use. 

We will represent arguments and relations between them through the means of argument graphs as defined by Dung \cite{Dung95}, 
which do not presuppose any particular argument structure and focus on modelling the attack relation.  

\begin{definition}
An {\bf argument graph} is a pair $\graph = ({\cal A},{\cal R})$ where 
${\cal A}$ is a set of \textbf{arguments} and ${\cal R}\subseteq {\cal A}\times {\cal A}$ represents a binary \textbf{attack relation} 
between arguments. 
\end{definition}

For arguments $A_i, A_j \in {\cal A}$, $(A_i,A_j) \in {\cal R}$ means that $A_i$ \textbf{attacks} $A_j$ 
(accordingly, $A_i$ is said to be a {\bf counterargument for} $A_j$). 
With $\attackers(A) = \{B \mid (B,A) \in \arcs(\graph)\}$ we will denote the set of attackers of an argument $A$.
In general, we 
say that an argument $A_i$ \textbf{(indirectly) attacks} an argument $A_j$ if there is a (directed) path of odd length from 
$A_i$ to $A_j$ in $\graph$. 
In a dual fashion, $A_i$ \textbf{defends} $A_j$ against $A_k$ if $(A_i, A_k) \in {\cal R}$ and $(A_k, A_j) \in {\cal R}$, 
and $A_i$ \textbf{(indirectly) defends} $A_j$ if there is a (directed) path of non-zero even length from 
$A_i$ to $A_j$ in $\graph$. With $\initarg(\graph) = \{A \mid \attackers(A) = \emptyset\}$ we denote the set of initial arguments of a graph, i.e. arguments that are not attacked at all. 


An argument graph can be easily depicted as a directed graph, where nodes represent arguments and arcs represent attacks.
We will therefore use $\nodes(\graph)$ to denote the nodes 
in $\graph$ (\emph{i.e.} $\nodes(\graph) = {\cal A}$)
and $\arcs(\graph)$ to denote the set of arcs 
in $\graph$ (\emph{i.e.} $\arcs(\graph) = {\cal R}$). 
An example of an argument graph is in Figure \ref{fig:intro1}.

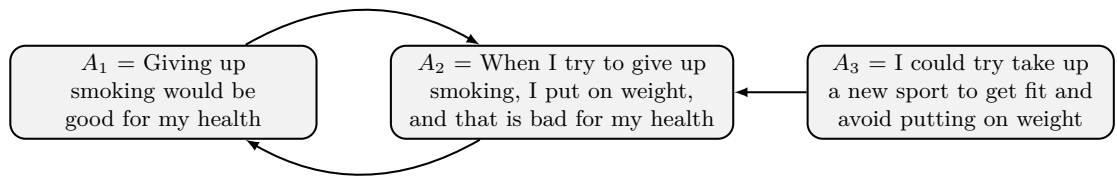
\begin{figure}[ht]
\centering
\resizebox{\textwidth}{!}{
\begin{tikzpicture}[->,>=latex,thick]
\tikzstyle{every node}=[{draw,text centered, shape=rectangle, rounded corners=6pt, fill=gray!10,font=\small}]
			\node (a1) [text width = 40mm] {$A_1$ = Giving up smoking would be good for my health};
			\node (a2) [right=of a1,text width = 45mm] {$A_2$ = When I try to give up smoking, I put on weight, and that is bad for my health};
			\node (a3) [right=of a2,text width = 40mm] {$A_3$ = I could try take up a new sport to get fit and avoid putting on weight};
			\path	(a1)[bend left] edge (a2);
			\path	(a2)[bend left] edge  (a1);
			\path	(a3) edge (a2);
\end{tikzpicture} 
}
\caption{\label{fig:intro1}
An example of an argument graph.
Here, we assume that $A_1$ and $A_2$ attack each other because they rebut each other, and we assume that $A_3$ attacks $A_2$ because it provides a counterargument to $A_2$. 
}
\end{figure}

Given an argument graph, a natural question to ask is which arguments are acceptable, \emph{i.e.}, which arguments can be seen as 
\enquote{winning}. Based on dialectical notions, Dung made some important proposals for acceptable subsets of arguments where each subset is conflict-free (\emph{i.e.}, no arguments in the set attacks another argument in the subset) and defensible (\emph{i.e.}, 
for every attacker of an argument in the subset, there is a defender in the subset). 
Numerous proposals have been made that investigate variants and developments of Dung's proposal and we refer the reader to \cite{HOFA2018} for a comprehensive review.  

In our work, we will use the definition of an argument graph. However, as we will see in Section \ref{sec:strategies}, we do not assume that the agents use dialectical semantics.
The reason for not using dialectical semantics is that we are not concerned with determining the arguments acceptable according to normative principles. Instead, we wish to model how persuasion may occur in scenarios where the participants are allowed the freedom
of opinion and as such do not need to adhere to any rationality principles. 
Certain studies show that the performance of dialectical semantics can be rather low in such applications \cite{PolbergHunter2018ijar,Rosenfeld2016}. Since we wish to construct a predictive model, we will not impose conditions for when an agent should be persuaded, but rather have a model that reflects how an agent is likely to behave. We refer to 
Section \ref{section:literaturereview} for a discussion on approaches that follow the dialectical semantics.

\section{User modelling}
\label{section:usermodelling}

In this section we will describe the type of user model that we want to incorporate into our APS. We focus on two possible 
dimensions - the concerns and the beliefs of the user -  and in the next sections we explain how they can be interpreted 
and modelled. 

\subsection{Concerns}
\label{sub:concerns}

A concern is meant to represent something that is important to an agent.
It may be something that they want to maintain (for example, a student may wish to remain \emph{healthy} during the exam period), or it may be something that they want to bring about (for example, a student may wish to \emph{do well in the exams}).
Often arguments can be seen as either raising a concern or addressing a concern, as we illustrate in the following example.

\begin{example}
\label{ex:concerns}
Consider the following arguments about abolishing student fees. Depending on the participant, the first argument could be addressing the concern of student finances whereas the second could be raising the concern of education.
\begin{itemize}
\item ($A_1$) The charging of fees should be abolished as it leads to students graduating with massive debts.
\item ($A_2$) The charging of fees should be abolished as it leads to universities investing leisure facilities rather than education facilities. 
\end{itemize}
\end{example}

The types of concern may reflect the possible motivations, agenda, or plans that the persuadee has in the domain (\emph{i.e.}, subject area) of the persuasion dialogue. 
They may also reflect the worries or issues she might have in the domain. 
When an argument is labelled with a type of concern, it is meant to denote that the argument has an impact on that concern, irrespective of the exact nature of that impact.

Various agents can, independently of each other, identify similar concerns. Thus, it may be appropriate to group these into types of concern. For example, we could choose to define the type \enquote{Fitness} to cover a variety of concepts, from 
lack of exercise through to training for marathons. 
The actual types of concern we might consider, and the scope and granularity of them, depend on the application. 
However, we assume that they are atomic, and that ideally, the set is sufficient to be able to type all the possible arguments that we might want to consider using in the dialogue. We therefore introduce the following notation:

\begin{definition}
Let $\concerndomain$ be the domain of concerns and $\args$ the set of arguments. 
The function $\concerns^{\dusertag}: \args \rightarrow 2^{\concerndomain}$ 
represents the assignment of concerns 
to arguments by the agent $\dusertag$. 
\end{definition} 
Where clear from the context, we can drop the superscript from the concern assignment.



While certain agents can agree on what kinds of concern a given argument raises or addresses, this does not mean that 
the concerns themselves are equally important or relevant to them. 
We thus also require information about the preferences over types of concern of an agent. 
For instance, if we have a collection of arguments on the topic of the university fees, we may have types such as \enquote{Student Well-being}, \enquote{Education}, or  
\enquote{Student Satisfaction}. We then may have an agent that regards 
\enquote{Student Well-Being} as the most important concern for them, \enquote{Education} being the second, and \enquote{Student Satisfaction} as last one. 
Another agent may have an entirely different preference regarding those categories, and these need to be represented, 
so that they can be harnessed by an APS to put forward more convincing arguments during a dialogue (see also Example 
\ref{ex:concern}).

\begin{definition}
With $\conpref_{\concerndomain}^{\dusertag}$ we denote the preference relation of the agent $\dusertag$ over concerns $\concerndomain$.
\end{definition}
By $C'\conpref_{\concerndomain}^{\dusertag} C$ we understand that $C$ is at least as preferred as $C'$ by agent $\dusertag$.  
When clear from the context, we can drop the super/subscripts to improve readability. 

There are some potentially important choices for how we can model agent's preferences.  
For instance, we can define them as pairwise choices or assume they form a partial or even 
linear ordering. 
While during the experiments we will focus on the linear approach 
(see Section \ref{section:datascience}), our general method is agnostic as to how preferences are represented.  
There are also numerous techniques for acquiring preferences from participants (see \emph{e.g.}, \cite{ChenPu04}, for a survey). 
We thus assume that appropriate representation and sourcing techniques can be harnessed depending on the desired application. 
 





\subsection{Beliefs}
\label{sub:beliefs}

The beliefs of the user strongly affect how they are going to react to persuasion attempts \cite{ogden2012health}. 
There is a close relationship between the belief an agent has in an argument, and the degree to which the agents regard the argument 
as convincing \cite{HunterPolberg17ictai}. Furthermore, beliefs 
can be used as a proxy for fine-grained argument acceptability, 
the need for which was highlighted by empirical studies conducted in \cite{Rahwan2011,PolbergHunter2018ijar}. 
We therefore treat the belief in arguments as a key dimension in a user model. 
In this section we explain 
how we represent belief for an individual, and how we can capture the uncertainty in that belief when considering multiple 
(sub)populations of agents.




For modelling the beliefs of a user, we use the epistemic approach to probabilistic argumentation \cite{Thimm14,Hunter2013ijar,BaroniGiacominVicig14,HunterThimm2017}, which defines 
a belief model as a \emph{probability distribution} over all possible subsets of arguments.

\begin{definition}
\label{def:beliefdistribution}
A {\bf probability distribution} over a graph $\graph$ is a function $\prob: 2^{\nodes(\graph)} \rightarrow [0,1]$ 
s.t. $\sum_{X \subseteq \nodes(\graph)} \prob(X) = 1$.
The {\bf belief in an argument} $A \in \nodes(\graph)$, denoted $\prob(A)$, is defined as:
    $$\prob(A) = \sum_{X \subseteq \nodes(\graph) \mbox{ s.t. } A \in X} \prob(X).$$
\end{definition} 

For a probability distribution $\prob$ and $A \in \nodes(\graph)$, the belief $\prob(A)$ that an agent has in $A$ is seen as 
the degree to which the agent believes $A$ is true.
When $\prob(A) > 0.5$, we say that the agent believes the argument to some degree, whereas when $\prob(A) < 0.5$, the agent disbelieves the argument to some degree. $\prob(A) = 0.5$ means that the agent neither believes nor disbelieves the argument. 
We would like to highlight that $\prob(A)$ (\emph{i.e.} the belief the agent has in argument $A$) is related to, but distinct from 
$\prob(\{A\})$ (\emph{i.e.} the probability assigned to set $\{A\}$). 

The persuader uses a belief distribution $\prob$ as a belief model of the persuadee and updates it at each stage of the dialogue 
in order to reflect the changes in persuadee's opinions. 
There are various possible ways to perform the updates (such as discussed in \emph{e.g.}, \cite{Hunter2015ijcai,Hunter2016sum}) and the method we will focus on will 
be explained in Section \ref{subsub:updating}.

 
Definition \ref{def:beliefdistribution} considers the belief we have in an argument. However, it lacks any quantification of the uncertainty about the belief in an argument. For example, an agent may be certain of the value that $\prob(A)$ takes for an argument $A$, or she may have some uncertainty associated with the assignment, with an extreme being when she is simply not sure of anything and $\prob(A)$ could take on any value in the unit interval. 
Furthermore, different agents may have different assignments when asked about their beliefs; or different reactions to the way the queries are formulated depending, for instance, on their personality (see, \emph{e.g.}, \cite{JohnsonHersheyMeszarosKunreuther93, SieglerOpfer03}).
This means that when we want to represent the probability distribution for a set of agents (to be harnessed in a user model), there is uncertainty of the value to choose for each argument. 
To address this, we use a proposal for constructing user models that is based on beta distributions \cite{HadouxHunter2018aamas}.

These distributions offer a well-established and well-understood approach to quantifying uncertainty.
Additionally, they allow for a principled way of representing subgroups within a population.
This is particularly important for applications in persuasion, where different subpopulations may have significantly different beliefs in the arguments in a dialogue. Furthermore, they may also have radically different ways of responding to specific dialogue moves.
In such situations, 
easy-to-get or already gathered data (such as a medical record for instance) can be leveraged to match a new user with a particular subpopulation in order to use a more efficient argumentation strategy.

So rather than use a probability distribution as given in Definition \ref{def:beliefdistribution} to formalize the belief in an argument, we will use a beta distribution. In a beta distribution for an argument $A$, the X axis gives the belief in the argument in the unit interval (i.e. $\prob(A)$) and the Y axis gives the probability density for that being the belief in the argument. So for a particular $x$, it is the probability of $\prob(A)$ being $x$. 

As we see in the following definition, the shape of the beta distribution is determined by two hyperparameters $\alpha$ and $\beta$. Figure \ref{fig:betaex} shows two examples of beta distributions with parameters, from left to right, $(\alpha=0.12, \beta=0.45)$ and $(\alpha=3.41, \beta=3.38)$.




\begin{definition}
\label{def:betadistribution}
A \textbf{beta distribution} $\mathcal{B}(\alpha,\beta)$ of parameters $\alpha$ and $\beta$ is a probability distribution defined as follows such that $x \in [0,1]$
\[
f(x,\alpha,\beta) 
= \frac{1}{B(\alpha,\beta)}x^{\alpha-1}(1 - x)^{\beta-1}
\]
where
\[
B(\alpha,\beta) =
\int_{0}^{1} x^{\alpha-1}(1-x)^{\beta-1} dx
\]
\end{definition}

Whilst this definition may appear complex, it gives a natural way of capturing the probability of a probability value. Furthermore, the definition can be easily understood in terms of capturing Bernoulli trials.

Given the definition for a beta distribution, the mean $\mu$ and variance $\nu$ can easily be obtained as follows.
\[
\mu = \frac{\alpha}{\alpha+\beta} 
\hspace{1cm}
\nu = \frac{\alpha\beta}{(\alpha+\beta)^2(\alpha+\beta+1)}
\]



Using the beta distributions gives us a number of advantages.
The distribution can handle the uncertainty on the belief (i.e. the uncertainty over the value assigned to $\prob(A)$ for an argument $A$)
whether it comes from the lack of prior knowledge or from the discrepancies in the cognitive evaluation of the belief.
It is also well suited to representing populations.
So if we have some data about the belief in an argument $A$ (for instance, if we ask some people for their belief in $A$ --- \emph{i.e.} their value for $\prob(A)$), we could have a sequence of values such as 0.6, 0.5, 0.6, 0.7, 0.6, 0.7 as data. From this, we can calculate a mean and variance for the data, denoted $\hat{\mu}$ and $\hat{\nu}$ respectively, and then it is straightforward to use these to estimate the $\alpha$ and $\beta$ values using the method of moments as follows, where $\hat{\alpha}$ and $\hat{\beta}$ are the estimates for $\alpha$ and $\beta$ respectively.

\[
\hat{\alpha} = \hat{\mu}\left(\frac{\hat{\mu}(1 - \hat{\mu})}{\hat{\nu}} - 1\right), 
\mbox{ if } \hat{\nu} < \hat{\mu}(1 - \hat{\mu})
\] 
\[
\hat{\beta} = (1 - \hat{\mu})\left(\frac{\hat{\mu}(1 - \hat{\mu})}{\hat{\nu}} - 1\right), 
\mbox{ if } \hat{\nu} < \hat{\mu}(1 - \hat{\mu})
\]

We can then plug the estimates $\hat{\alpha}$ and $\hat{\beta}$ into Definition \ref{def:betadistribution} to get a beta distribution that is an estimate of the beta distribution for the population.

Another advantage of beta distributions is that we use them to detect the subpopulations with homogeneous behaviours (\emph{i.e.} similar belief). In other words, we may find that the data about a population suggests that there are multiple underlying beta distributions. This can be handled by the notion of a mixture of beta distributions which we define below.

\begin{figure}[t]
  \centering
  \includegraphics[scale=0.65]{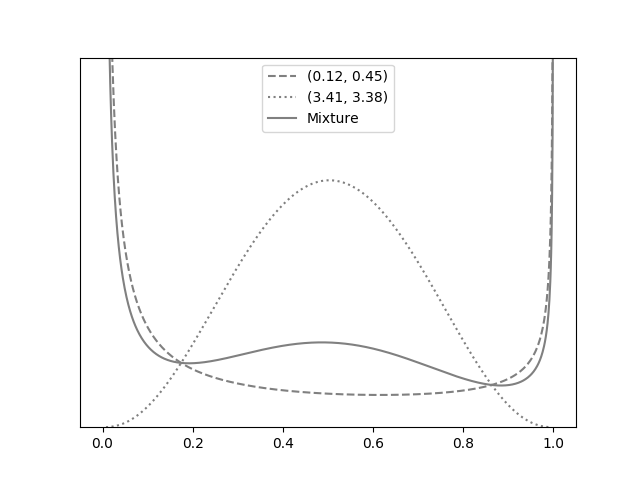}
  \caption{Examples of beta distributions with parameters $(\alpha=0.12, \beta=0.45), (\alpha=3.41, \beta=3.38)$ (zoomed in for visualisation) for argument \emph{\enquote{University education is an investment in the economy of the whole country, and therefore everyone should contribute to university education.}}}
  \label{fig:betaex}
\end{figure}


The mixture of beta distributions in Figure \ref{fig:betaex} shows the initial belief of all the participants in the argument \emph{\enquote{University education is an investment in the economy of the whole country, and therefore everyone should contribute to university education.}}.
We see that a unimodal distribution (\emph{i.e.}, containing only one \enquote{bell}) cannot accurately represent the data.
Indeed, it is composed of extreme values on both ends, and high values in the middle of the range. This multimodal shape suggests that there is a heterogeneous underlying population. Therefore, there are multiple underlying beta distributions, each representing a more homogeneous subpopulation.
For this reason, we use a mixture of beta distributions in order to create a multimodal distribution.
Each distribution is called a component, and all components are weighted and summed as a linear combination as defined next.

\begin{definition}
A \textbf{beta mixture} is characterised by a triple $\langle \bm{\alpha} = (\alpha)_{\{1, \ldots, C\}}$, $ \bm{\beta} = (\beta)_{\{1, \ldots, C\}}$, $ \bm{\pi} = (\pi)_{\{1, \ldots, C\}} \rangle$ where $C$ is the number of components, $\alpha_c, \beta_c$ and $\pi_c$ are respectively the parameters $\alpha, \beta$ and the weight associated with component $c \in \{1,\ldots,C\}$.
\end{definition}
Therefore, a mixture $M$ is calculated as follows:
$$
M(\bm{\alpha}, \bm{\beta}, \bm{\pi}) = \sum_{c=1}^C \pi_c \times \mathcal{B}(\alpha_c, \beta_c).
$$

By extension, the probability of a belief $x \in [0,1]$
(i.e. an assignment for $\prob(A)$ for an argument $A$, which is also called a sample when talking about a value in a dataset) under the mixture $M(\bm{\alpha, \bm{\beta}, \bm{\pi}})$ is:
$$
M(x; \bm{\alpha}, \bm{\beta}, \bm{\pi}) = \sum_{c=1}^C \pi_c \times \mathcal{B}(x; \alpha_c, \beta_c)
$$
where $\mathcal{B}(x;\alpha_c,\beta_c)$ is the function giving the probability that sample $x$ has been drawn from the beta distribution of parameters $\alpha_c$ and $\beta_c$.

Figure \ref{fig:betaex} also presents the mixture of the two components, combined with a weights vector $\pi = [0.15, 0.85]$.
Therefore, the probability of a belief $x$ with the mixture $M$ in Figure \ref{fig:betaex} is: 
$$P_M(x) = 0.15 \times \mathcal{B}(x;0.12, 0.45) + 0.85 \times \mathcal{B}(x;3.41,3.38).$$

So beta distributions offer a flexible and practical way of representing the belief in arguments when there is uncertainty in what that belief might be. As we will see, we can collect data from crowdsourced participants about their belief in arguments, and use this to populate the beta distributions in our user models. This involves finding the choice of components that best describes the data. This may involve a trade-off of the number of components and the fit with the data (see \cite{HadouxHunter2018aamas} for more details). 






\section{Dialogue representation and rules} 
\label{sec:dialogues}

There are many possible moves and protocols for dialogical argumentation (see Section \ref{section:literaturereview}), and they 
vary in their goals or properties. For the purpose of this study, the protocol we require would have to meet the following 
principles - \textbf{asymmetry}, \textbf{timeliness} and \textbf{incompleteness}. 
 
\textbf{Asymmetry} of a dialogue means that different parties have different types of moves available, or one 
party has control over how others can express their opinions. In our case, this means that while the system can put forward 
arguments freely, the user will need to select theirs from a list displayed to them, referred to as a menu.  
This bypasses the need for a natural language processing module within our APSs that would extract user arguments 
from free text, but puts additional burden on the domain design to increase the applicability of the menus to the users.  

\textbf{Timeliness} means that the user's counterarguments are addressed when they are stated. Typically, persuasion 
dialogues allow only for one argument to be expressed at a time, or proceed in a depth-first manner. In other words, if 
the user has three counterarguments, then the first one (and its counterarguments, and counterarguments of these counterarguments, etc) is fully explored before moving onto the second. This kind of narrative can be natural 
for various settings, such as scientific discussions, but less intuitive in others. It also makes the success of the dialogue 
vulnerable to the estimation of what the user's most important issue is (\emph{i.e.} which discussion branch should 
be explored first). Getting the order right is even more important due to the 
fact that the longer the discussion is, the less effective it may be \cite{TanNDNML16}.
Consequently, our preference would be to investigate a more breadth-first approach, where user's issues 
are not \enquote{put on hold}, which can mean that multiple arguments may have to be dealt with at each step of the dialogue. 
 
Last, but not least, \textbf{incompleteness} means we do not force all arguments to be matched with 
appropriate counterarguments. In real life, an exhaustive response can simply be too exhausting to the user, 
which creates the risk of disengagement. In some cases, dealing with key issues may be more effective than 
dealing with all of them. We are therefore more concerned with what are the allowed exchanges between the system and user 
and what effect they have. Completeness of a dialogue is often an effect of assuming that a dialogue is evaluated 
using dialectical semantics, which is something we are not doing in this study. 

In the next sections we will provide a formalization of dialogue moves and the above intuitions concerning the protocol. 
In this paper we will assume that a dialogue is a sequence of moves $\dialogue = [m_1,\ldots,m_k]$. 
Equivalently, we use $\dialogue$ as a function with an index position $i$ to return the move 
at that index (\emph{i.e.}, $\dialogue(i) = m_i$). With $\args(m_i)$ we will denote the set of arguments 
expressed in a given dialogue move; the precise definition will be given when we define the moves. 
Finally, with ${\sf Length}(\dialogue) = k$ we will denote the length of a given dialogue.

\subsection{Dialogue moves}
\label{section:dialoguemoves}



In this paper, we will only consider two types of move - the posit move, which will be used by the APS to state 
arguments, and the menu move, which will allow the user to give their counterarguments. 
In order to make certain notation easier, throughout this section we will assume that we have 
a dialogue $\dialogue$ and a graph $\graph$ s.t. the arguments expressed in the dialogue come from 
this graph. In other words, we assume that for every step $i$, $\args(\dialogue(i)) \subseteq \nodes(\graph)$.

The purpose of the \textbf{posit move} is to present  the starting argument, and after that the counterarguments to the previously presented arguments. 
In order to do so, we first consider a function returning a set of attackers of a given argument that does not contain 
statements that have already been made during a dialogue.   
 
\begin{definition}
The set of {\bf options} of argument $A$ at step $i$ in dialogue $\dialogue$ is defined as follows, where for all $j < i$, 
$\args(\dialogue(j)) \subseteq \nodes(\graph)$.
\[
\options(\dialogue,A,i) = \{ B \mid B \in \attackers(A) \mbox{ and there is no } j < i \mbox{ s.t. } B \in \args(\dialogue(j)) \}
\]
\end{definition}

With this, we define posit moves as follows: 

\begin{definition} 
For a dialogue $\dialogue$ and step $i$, the set of {\bf posit moves} is given by a function $\positmoves(\dialogue,i)$, which is 
defined as follows:
$$\positmoves(\dialogue,i) =
\begin{cases}
\{ \{A\} \mid A \in \nodes(\graph)\} & \text{ if } i=1 \\
\{ X \mid X \subseteq \bigcup_{A \in \args(\dialogue(i-1))} \options(\dialogue,A,i)\} & \text{ otherwise }\\
\end{cases}
$$ 
\end{definition}

We note that for every posit move $X \in \positmoves(\dialogue,i)$, $\args(X) = X$. We also observe 
that we do not force $X$ to contain a counterargument for every argument from the previous for $i>1$, 
thus creating the possibility of not countering certain arguments. In fact, $X$ may be an empty set, representing 
a situation where the system may have no arguments left to utter, or a scenario in which not saying anything 
may be the best outcome \cite{PaglieriCastel10,Paglieri2009}.

Now we introduce the \textbf{menu move} as a way for the user to give their input into the discussion. 
In an asymmetric dialogue, the counterarguments to choose from are displayed to the user by the system, 
and the user is meant to select their response from the list \footnote{Please note that displaying the listing does not count as a dialogue move in this approach. It is merely a way to facilitate the user move.
}
In our approach, we not only 
include the arguments to choose from, but also  the null options 
(denoted with $\nacc_A$ and $\nrej_A$) which mean that the user agrees with $A$ and does not give any 
counterarguments, or disagrees with $A$ but none of the listed arguments are applicable. 
An example of this is seen in 
Figure \ref{fig:interface}, and can be formalized as follows:

\begin{definition}
The {\bf menu listing} at step $i$ in dialogue $\dialogue$ for an argument $A$ 
is defined as follows:
\[
\menulisting(\dialogue,A,i) =
\begin{cases}
\emptyset & \attackers(A) = \emptyset \\
\options(\dialogue,A,i) \cup \{ \nrej_A, \nacc_A \} & \text{otherwise}\\
\end{cases}
\]
\end{definition}

\begin{figure}[t]
\centering
{\small
\begin{tabular}{|r l r|}
\hline
 \multicolumn{3}{|p{9cm}|}{\large \vspace{0.1cm}\textbf{Select your reason(s) against each statement presented below.}}\\
 &&\\
  \multicolumn{3}{|l|}{\normalsize Universities should continue charging students the 9K fee.}\\
  \multicolumn{3}{|l|}{\normalsize \textit{I disagree with the statement because:}}\\
  &&\\
  $\Box$ & \multicolumn{2}{p{8cm}|}{Student fees should be abolished because they are unfair.}\\
  $\Box$ & \multicolumn{2}{p{8cm}|}{Education in the UK used to be free in the past and we should continue that model.}\\
  $\Box$& \multicolumn{2}{p{8cm}|}{University education is becoming too expensive and many students end up in debts, hence the 9K fees should be abolished.}\\
  $\Box$ & \multicolumn{2}{p{8cm}|}{Student fees have a negative impact on the education quality and experience and should be abolished.}\\
  $\Box$ & \multicolumn{2}{p{8cm}|}{None apply to me and I agree with the statement.}\\
  $\Box$ & \multicolumn{2}{p{8cm}|}{None apply to me and I disagree with the statement.}\\  
  & & \textcolor{blue}{SEND} \\
  &&\\
 \hline
\end{tabular}
}
\caption{Interface for an asymmetric dialogue move for asking the user's counterarguments. Multiple statements (and their counterarguments) can be displayed, one after another.
}
\label{fig:interface}
\end{figure}


A menu move is simply a selection of responses from the menu listings against previously stated arguments, with the 
added constraints that for any given argument, a choice has to be made, and one cannot pick the null moves and counterarguments at the same time. 

\begin{definition}
For a dialogue $\dialogue$ and step $i$, the set of {\bf menu moves}  is given by a function $\menumoves(\dialogue,i)$, which is 
defined as follows and where it is assumed that for all $j < i$, $\args(\dialogue(j)) \subseteq \nodes(\graph)$:
\[
\menumoves(\dialogue,i) =  \{X_1 \cup \ldots \cup X_l \mid \forall_{1\leq h \leq l} \, X_h \in \menupick(\dialogue, A_h, i)\}
\]
where $\{A_1, \ldots, A_l \}$ is the set of all arguments in $\args(\dialogue(i-1))$ that posits an attacker in $\graph$ 
and 
\[\menupick(\dialogue, A, i) =  
\{X \mid \emptyset \neq X \subseteq \menulisting(\dialogue,A,i) \cap \nodes(\graph) \text{ or } X \in \{ \{\nrej_A\}, \{\nacc_A\} \}
\] 
\end{definition}

We observe that for a given menu move $X \in \menumoves(\dialogue,i)$, the set of presented arguments is defined as 
$\args(X) = X \cap \nodes(\graph)$.  
 
The posit and menu moves as defined here are only some of the possible moves that could be 
used in asymmetric persuasion dialogues. Nevertheless, they are sufficient for our purposes and have the benefit 
of resembling other popular non-dialogical interfaces for user input. In the next sections, we define our dialogue 
protocol based on these moves. 

\subsection{Dialogue protocol}
\label{section:protocol}

Let us now formally define the protocol for our dialogues using the previously proposed moves. 
The use of posit and menu moves and the lack of certain restrictions on them allows the protocol to 
meet our principles of asymmetry, timeliness and incompleteness. This, along 
with certain classical requirements (such as participants taking turns in expressing their opinions), leads 
to the following formalization. 
 
\begin{definition}
\label{def:protocol}
A dialogue $\dialogue = [m_1,\ldots,m_k]$ adheres to our \textbf{incomplete asymmetric dialogue protocol} 
on an argument graph $\graph$ if it satisfies
the following conditions, where $Q$ is the assignment of a participant to each step:
\begin{enumerate}
\item For each step $i \in \{1,\ldots,k\}$, $\args(m_i) \subseteq \nodes(\graph)$.

\item For each step $i \in \{1,\ldots,k\}$, if $i$ is odd, then $Q(i) = \dsystem$, else $Q(i) = \duser$.

\item For step $i = 1$, $m_i = \{A\}$ where $A$ is the persuasion goal.  

\item For each step $i$ such that $2 \leq i \leq k$, if $Q(i) = \duser$,
then $m_i \in \menumoves(\dialogue,i)$.

\item For each step $i$ such that $3 < i \leq k$, if $Q(i) = \dsystem$, then $m_i \in \positmoves(\dialogue,i)$.

\item For step $i = 3$, 
for every $A \in \args(\dialogue(i-1))$, there is $B \in \args(m_i)$ s.t. $B \in \attackers(A)$. 
\item For each step $i \geq 5$ s.t. $Q(i) = \dsystem$, 
$|\args(m_i) \setminus \initarg(\graph)|\leq 2$.
\item For the final step $k$, one of the following conditions hold, and for all steps $j < k$, neither of the conditions holds.
\begin{enumerate}
\item $m_k = \emptyset$ and $Q_k = \dsystem$, or
\item $\menumoves(\dialogue, k) = \emptyset$ and $Q_k = \duser$.
\end{enumerate}
\end{enumerate}
\end{definition}  
 
\begin{example} 

For the purpose of our experiments, we have constructed two argument graphs that we will discuss 
in Section \ref{sec:experimentgraphs}. 
Figure \ref{fig:keepingexcerpt} presents 
the subgraph of the original graph associated with the discussion visible 
in Table \ref{tab:protocolexample}. The table presents a dialogue 
between the user and one of the APSs we have implemented; for the sake of readability, we refer to arguments with their 
tags rather than textual content. 

We observe that the agents take turns in presenting their arguments, and since the definition of our dialogue moves 
forces an argument-counterargument relation, the system moves are at an even distance from the persuasion goal 
(argument $0$) and user moves are at an odd distance. In other words, they are (indirect) defends and (indirect) attackers 
of the goal respectively. The fact that not every user argument has to be answered by the system in our protocol results
in arguments $17$, $34$, $35$, $36$ and $83$ being unattacked. 
\end{example}

\begin{table}[t]
\centering
\begin{tabular}{ccc}
\toprule
Step & Agent & Move Made\\
\midrule
1 & \dsystem & $\{0\}$ \\
2 & \duser& $\{1, 2, 3, 4\}$ \\
3 &\dsystem & $\{5, 10, 12, 15\}$ \\
4 & \duser & $\{16, 17, 18, 28, 34, 35, 36, 37\}$\\
5 & \dsystem & $\{40, 55, 70, 71\}$ \\
6 & \duser & $\{81, 83, 93\}$\\
7 & \dsystem & $\{100, 113\}$ \\
\bottomrule
\end{tabular}
\caption{
An example of a dialogue adhering to the asymmetric posit protocol for the argument graph in favour of maintaining 
student fees (see Data Appendix). 
}
\label{tab:protocolexample}
\end{table}

\begin{figure}[ht!]
 \centering
\includegraphics[width=\textwidth]{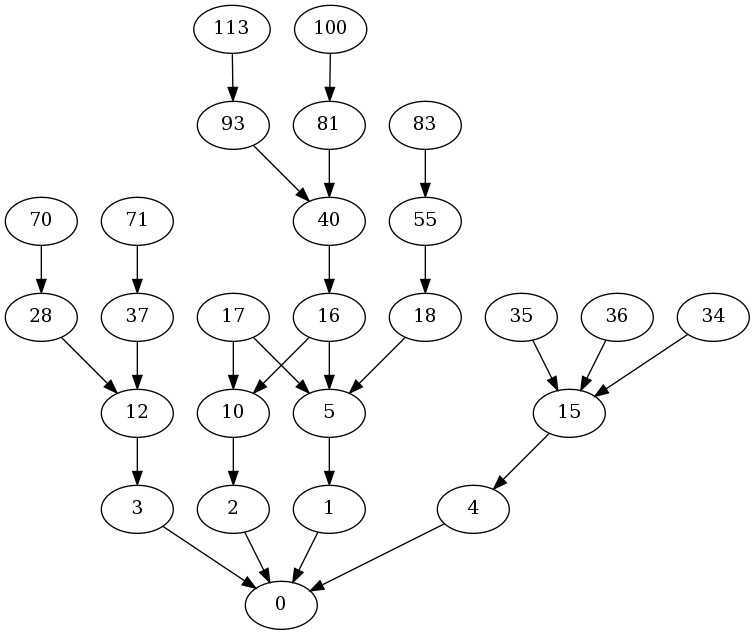}  
\caption{A subgraph of the argument graph in favour of maintaining 
student fees (see Data Appendix) induced by the dialogue from Table \ref{tab:protocolexample}.}
 \label{fig:keepingexcerpt}
\end{figure}

The above restrictions can be simply explained as follows. 
First, only the arguments that occur in the graph can be exchanged. 
The system and the user take turns in the dialogue, 
with the system making the first move by positing the persuasion goal. The system posits need to be met 
with appropriate user menu moves, and the counterarguments raised by the user may be (not necessarily fully) addressed by system posits. The system is only forced to give a complete response to the first user move - after that, the responses can be partial. In particular, we limit the number of active dialogue lines to two, i.e., that at most two arguments that can still be responded to by the user can be played by the system. This also means that any number of initial arguments can be played, as they do not lead to further discussions.
Finally, the dialogue terminates if  the system decides not to play any arguments or no further moves can be made. 
This can happen if neither the user nor the system 
have any arguments left to play (\emph{i.e} we have reached leaf arguments), the user concedes or all of their counterarguments 
are outside of the domain (\emph{i.e.} the user chooses only the $\nacc$ or $\nrej$ responses).

At the end of the dialogue, the system (\emph{i.e.}, the persuader) hopes to have convinced the user (\emph{i.e.}, the persuadee) to accept the persuasion goal (\emph{i.e.}, the first argument in the dialogue). 
We do not use dialectical semantics to determine whether the persuasion goal is a winning argument (for example, if it is in a grounded or preferred extension of the subgraph of $\graph$ induced by the arguments that appear in the dialogue). 
In our APSs, the dialogue will be evaluated by asking the user whether they believe the persuasion goal and to what degree
after the discussion has ended (see Section \ref{section:experiments}).

Our approach is one of many, and 
there exist various different dialogue protocols (see also Section \ref{section:literaturereview}). Nevertheless, 
we are not aware of other methods that would adhere to our principles.
This protocol is different to the dialogue protocols for abstract argumentation that are used for determining whether specific arguments are in the grounded extension \cite{Prakken2005} or preferred extension \cite{CaminadaPodlaszewki12}. 
It is also different to the dialogue protocols for arguments that are generated from logical knowledge bases (\emph{e.g.}, \cite{BH09,FanToni11}). 
Those protocols are concerned with determining the winning arguments in a dialogue in a way that is sound and complete with respect to the underlying knowledge base. Finally, it is worth noting that many proposals for dialogical argumentation protocols involve depth-first search (\emph{e.g.}, \cite{BenchCapon2002}), which goes against our timeliness requirement.  


\section{Dialogue strategies}
\label{sec:strategies}

Typically, at any step of the dialogue, there can be multiple move options to choose from. In other words, particularly 
with large domains, more often than not it can be the case that 
$|\positmoves(\dialogue,i)| >1$ for the system and $|\menumoves(\dialogue,i)| >1$ for the user. 
While the aim of the APS is to select an appropriate posit move out of the available ones, the estimation 
of user's actions can affect what \enquote{appropriate} is. Consequently, the APS needs some strategy on how to proceed, 
and in this section we will focus on two options - the MCTS-based and the baseline strategies. The first one is an advanced 
method, harnessing the information contained in the user model; the other one serves as a basis for comparison and represents 
an agent selecting random moves out of the available ones at each step of the dialogue. 

\subsection{Advanced strategy}
\label{sec:advancedstrategy}
 
Our approach to making strategic choices of move is to harness decision trees \cite{HadouxHunter17}.
A \textbf{decision tree} represents all the possible combinations of decisions and outcomes of a sequential decision-making problem.
In a situation with two agents taking turns, a path from the root to any leaf crosses alternately nodes associated with the proponent (called \emph{decision nodes}) and nodes associated with the opponent (called \emph{chance nodes}). In our case, the role 
of the proponent is played by the APS, and user is the opponent. 

In the case of dialogical argumentation, a full decision tree represents all the possible dialogues. Each path from the root to a leaf is one possible permutation of the moves permitted by the dialogue protocol \emph{i.e.}, one possible complete dialogue between the two agents.
An edge between any two nodes $n$ and $n'$ in the tree is the decision (\emph{i.e.}, dialogue move) that has to be taken by the corresponding agent in order to transition from node $n$ to node $n'$. We give an example in Figure \ref{fig:new} where, for the sake of readability, each move is the posit of a single argument (we note that the protocol allows for exchanging sets of arguments).   
Note that in our case the decision tree is from the point of view of the proponent.
Therefore, even if both the proponent and the opponent make decisions on the next argument to put forward, from the point of view of the proponent, only her moves are decisions (hence \emph{decision nodes}) and the opponent moves can only be predicted and later observed (hence \emph{chance nodes}).

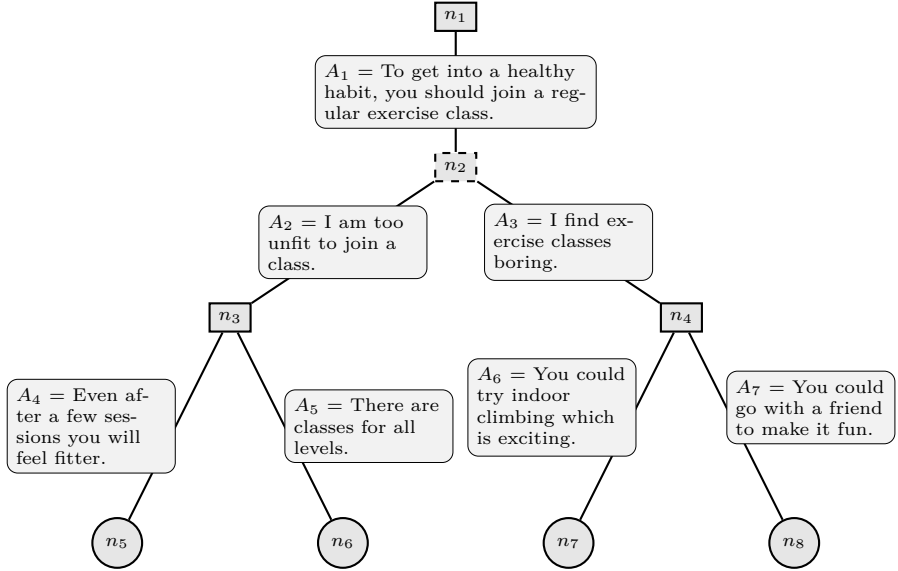
\begin{figure}[hbt]
\begin{center}
\begin{tikzpicture}[-,thick]
\node[draw,rectangle,fill=gray!20] (n1)  at (4.5,6) {\scriptsize $n_1$};
\node[draw,rectangle,dashed,fill=gray!20] (n2)  at (4.5,4) {\scriptsize $n_2$};
\node[draw,rectangle,fill=gray!20] (n3)  at (1.5,2) {\scriptsize $n_3$};
\node[draw,rectangle,fill=gray!20] (n4)  at (7.5,2) {\scriptsize $n_4$};
\node[draw,circle,fill=gray!20] (n5)  at (0,-1) {\scriptsize $n_5$};
\node[draw,circle,fill=gray!20] (n6)  at (3,-1) {\scriptsize $n_6$};
\node[draw,circle,fill=gray!20] (n7)  at (6,-1) {\scriptsize $n_7$};
\node[draw,circle,fill=gray!20] (n8)  at (9,-1) {\scriptsize $n_8$};
\path (n2)[] edge[] node[draw,rectangle,rounded corners,line width=0.3pt,text width=35mm,font=\scriptsize,fill=gray!10]{$A_1$ = To get into a healthy habit, you should join a regular exercise class.} (n1);
\path (n3)[] edge[] node[draw,rectangle,rounded corners,line width=0.3pt,text width=20mm,font=\scriptsize,fill=gray!10]{$A_2$ = I am too unfit to join a class.} (n2);
\path (n4)[] edge[] node[draw,rectangle,rounded corners,line width=0.3pt,text width=20mm,font=\scriptsize,fill=gray!10]{$A_3$ = I find exercise classes boring.} (n2);
\path (n5)[] edge[] node[draw,rectangle,rounded corners,line width=0.3pt,pos=0.5,left,text width=20mm,font=\scriptsize,fill=gray!10]{$A_4$ =  Even after a few sessions you will feel fitter.} (n3);
\path (n6)[] edge[] node[draw,rectangle,rounded corners,line width=0.3pt,pos=0.5,right,text width=20mm,font=\scriptsize,fill=gray!10]{$A_5$ = There are classes for all levels.} (n3);
\path (n7)[] edge[] node[draw,rectangle,rounded corners,line width=0.3pt,pos=0.6,left,text width=20mm,font=\scriptsize,fill=gray!10]{$A_6$ = You could try indoor climbing which is exciting.} (n4);
\path (n8)[] edge[] node[draw,rectangle,rounded corners,line width=0.3pt,pos=0.6,right,text width=20mm,font=\scriptsize,fill=gray!10]{$A_7$ = You could go with a friend to make it fun.} (n4);
\end{tikzpicture}
\end{center}
\caption{\label{fig:new} A decision tree for an argumentation dialogue. Each arc is labelled with a posit move in a dialogue, which for readability purposes is assumed to consist of only single arguments in this example. Each branch denotes a dialogue involving exactly three arguments with the first (respectively the second) being posited by the proponent (respectively the opponent). The proponent (decision) nodes are solid boxes, the opponent (chance) nodes are dashed boxes and the leaf nodes are circles.}
\end{figure}



In order to compare different dialogues so as to be able to select the best one that can be reached from each step, 
we need to define a {\bf reward function} that gives a value of the dialogue or outcome of the dialogue to the system. 
Every node in a tree can then be evaluated based on the outcomes it leads to. Hence, for every decision node, 
we can also find an action to perform (\emph{e.g.}, the arguments to posit in each state 
of the debate) that would lead to a more beneficial result according to a given criterion.

Decision trees are useful tools in artificial intelligence. However, they also have their limits, and quickly become unmanageable 
in applications with a large number of possible outcomes. For instance, while we can use decision trees for a tic-tac-toe game, 
Go is too complicated. Unfortunately, given the sizes of argument graphs we will be dealing with in this paper, the same holds 
for our APS. A possible solution is to make use of appropriate sampling techniques that explore only certain branches of 
the tree (\emph{i.e.} only certain dialogues) rather than all of them, and in this paper we will rely on the Monte Carlo Tree Search 
method.  
%
%
%

Monte-Carlo Tree Search (MCTS) \cite{Coulom07} methods are amongst the most efficient online methods to approximately solve large-sized sequential decision-making problems (for a review, see \cite{Browne2012}). This method is notably used in the \emph{Partially Observable Monte-Carlo Planning} (POMCP) algorithm \cite{SilverVeness10} and in applications such as Alpha-Go \cite{Silver16}. Unlike traditional decision tree solving methods such as backward induction, using an MCTS is significantly less 
affected by the dimensionality. Since the branching factor (\emph{i.e.} the number of actions we can perform in a given node) 
increases with the number of arguments we can play, choosing a resilient method is vital for the efficiency of our system.





\subsubsection{Monte Carlo Tree Search}
\label{section:montecarlo}

The approach can be roughly split into four phases - selection, expansion, simulation and backpropagation - that are repeated until the desired number of simulations has occurred or some time limit is reached. 

\begin{description}
\item[Selection:] Starting from the root of the tree (the current state of the dialogue) the algorithm chooses an action to perform in a black box simulator of the environment. It uses the UCB1 \cite{AuerCesaBianchiFischer02} procedure to choose the action and then observes the new state of the environment that is output by the simulator. It then goes down a level in the tree depending on this new state. The algorithm repeats this step until it reaches a leaf in the tree.
\item[Expansion:] If this leaf is not a terminal state of the problem (\emph{i.e.}, a possible end of the dialogue), the algorithm expands the tree at this leaf and adds a child for each possible subsequent state.
\item[Simulation:] Once the leaf node has been expanded (and is thus not a leaf anymore) the algorithm simulates all the subsequent steps in the dialogue until it reaches a possible terminal state. This simulation does not expand the tree.
\item[Backpropagation:] Once a terminal state has been reached, a reward can be calculated and then backpropagated up in the tree to calculate the most promising nodes.
\end{description}

These four steps are repeated until the desired number of simulations has occurred or some time limit is reached.
At this point, the most promising next argument in the dialogue is selected as the child of the root node with the highest backpropagated reward.
This argument is played in the real dialogue, a new state is observed after the user responds and the root of the simulation tree is moved down to the node representing this new state.

\subsubsection{Reward function}
\label{section:reward}

The purpose of the reward function is to be able to compare dialogues; the higher the reward, the better or more desirable 
the dialogue is. In our framework, the reward function is based on the usage of concerns arising in the arguments, 
and the belief in the persuasion goal at the end of the dialogue. 
For this, we have designed the reward function in two parts, that are ultimately combined into one single value, 
as we describe in this section. 

\paragraph{Scoring of concerns}\hfill
\label{section:scoringconcerns}


In this subsection, we will introduce a function for scoring a dialogue in terms of the concerns that are associated with the arguments presented in the dialogue, and the user's preferences over them. 
We will use this function to compare dialogues as part of a reward function in Section \ref{sec:rerwardfunction}. 
Note, in this paper, we assume that concerns and preferences between them are static and so not updated during the dialogue. 

The aim of the concern scoring function is to reflect how well the arguments posited by the system match the 
user's preferences over concerns. 
We assume that arguments covering fewer, but more important concerns to the user, 
are more interesting than arguments covering more, but less relevant concerns.


We thus aim to select the most 
appropriate argument(s) to state out of the possible ones.  
Every argument uttered in a dialogue (aside from the persuasion goal) 
is stated in response to one or more arguments that appeared in the previous move (see Section \ref{section:protocol}). 
We can thus speak about \enquote{dialogue parents} of a given argument (\emph{i.e.}, 
arguments at a previous step that they attack) 
and \enquote{dialogue siblings} (\emph{i.e.}, all arguments that could potentially be used against a dialogue parent). 
We can then analyze the concerns associated with a given argument, with its siblings, with those that have appeared in 
the dialogue and those that have not. 

\begin{notation}
We introduce the following notation, where $\concerns(A)$ denotes 
the concerns associated with an argument $A$.

%
%
%
%

\begin{tabu} to \textwidth{rcX}
$\concerns(\dialogue,i)$ &=& $\bigcup_{A \in \args(\dialogue(i))} \concerns(A)$\\
$\siblings(\dialogue,A)$ &=& $\{A' \mid \exists \, 1< i \leq {\sf Length}(\dialogue)$, $B \in \args(\dialogue(i-1))$ s.t. 
$A \in \args(\dialogue(i))$, $(A,B) \in \arcs(\graph)$ and $(A',B) \in \arcs(\graph)\}$\\
$\sibcon(\dialogue,A)$ &=& $\bigcup_{A' \in \siblings(\dialogue,A)} \concerns(A')$\\
$\sibcon(\dialogue,i)$ &=& $\bigcup_{A \in \args(\dialogue(i))} \bigcup_{A' \in \siblings(\dialogue,A)} \concerns(A')$\\
$\exsibcon(\dialogue,i)$ &=& $\sibcon(\dialogue,i) \setminus \concerns(\dialogue,i)$\\
\end{tabu}

\end{notation}

We explain these functions as follows:
$\concerns(\dialogue,i)$ gives the concerns of the arguments stated at the i-th step of a dialogue $\dialogue$; 
$\siblings(\dialogue, A)$ gives the dialogical siblings of an argument $A$, representing other arguments that attack the target of argument $A$ in the dialogue; 
$\sibcon(\dialogue,A)$ gives the concerns of the siblings of an argument $A$ in a dialogue $\dialogue$,
$\sibcon(\dialogue,i)$ gives the concerns of the siblings of all argument stated at the i-th step of a dialogue $\dialogue$;
and
$\exsibcon(\dialogue,i)$ gives the concerns of the siblings of all arguments stated at the i-th step of a dialogue $\dialogue$ excluding the concerns associated with the arguments that appeared at that step.


In addition, we require the function $\prefscore(C',C) \in [0,1]$ which states the proportion of population who prefer concern $C'$ over concern $C$. Crowdsourcing the preferences of participants is explained in Section \ref{section:preferenceconcerns}.

\begin{definition}
Let $\{\dusertag_i\}_{i=1}^n$ be a non-empty set of agents. For a given $C', C \in \concerndomain$,
$$\prefscore(C',C) = \frac{|\{\dusertag \mid \dusertag \in \{\dusertag_i\}_{i=1}^n, \, C \conpref_{\concerndomain}^{\dusertag} C'|}{n} = \frac{1}{n}\left(\sum_{i=1}^n \mathds{1}_{C \conpref_{\concerndomain}^{U_i} C'}\right)$$
where $\conpref_{\concerndomain}^{\dusertag}$ is the preference relation over concerns associated with agent $\dusertag$. 
\end{definition}

The concern score of a given dialogue step associated with the system 
is now defined in terms of how \enquote{good} or \enquote{bad} the sibling arguments 
that were not played are. For obvious reasons, the first step in which the persuasion goal is played is ignored. 
The score of the dialogue is then simply an average of these values:

\begin{definition}
Assume ${\sf Length}(\dialogue) = k \geq 3$ and let $n$ be the number of odd steps after the first step (\emph{i.e.} $n = \lceil (k/2) - 1 \rceil$).
The {\bf concern score} is 
\[
\concernscore(\dialogue) =
\frac{1}{n}
\sum_{i = 1}^{n} 
\bigg(
1 - \nonchosenscore(\dialogue,i)
\bigg)
\]
where the {\bf non-chosen} score is calculated as $\nonchosenscore(\dialogue,i)$ =
\[
\frac{1}{\mid\concerns(\dialogue,2i+1)\mid}
\times
\bigg(
\sum_{C \in \concerns(\dialogue,2i+1)} 
\sum_{C' \in \exsibcon(\dialogue,2i+1)}
\frac{\prefscore(C',C)}{\mid \exsibcon(\dialogue,2i+1)\mid}
\bigg)
\] 
\end{definition}


%

The non-chosen score for a step of the dialogue generates an average preference score for each non-chosen concern. This is done by taking each concern of the arguments played (\emph{i.e.}, $\concerns(\dialogue,2i+1)$) and each concern of the sibling arguments that are not played (\emph{i.e.}, $\exsibcon(\dialogue,2i+1)$). This average preference score is normalised by the number of concerns that appear in the chosen arguments. So the more that there is a concern $C'$ that is not played and a concern $C$ that is played and $C'$ is population-preferred to $C$ then the greater the non-chosen score is. 
This effect is normalised by the number of concerns raised by these non-played arguments (\emph{i.e.}, $\exsibcon(\dialogue,2i+1)$). This is to ensure that arguments are not favoured if they have many concerns associated with them. In other words, there is a bias in favour of arguments that have focused concerns.


%
%
%

We observe that the concern score of a dialogue is always in the $[0,1]$ interval. 

\begin{proposition}
For all dialogues $\dialogue$, $\concernscore(\dialogue) \in [0,1]$
\end{proposition}

\begin{proof}
Let $\mid\concerns(\dialogue,2i+1)\mid = x$
and let $\mid\exsibcon(\dialogue,2i+1)\mid = y$.
Since, for any concerns $C,C'$, $\prefscore(C',C) \in [0,1]$,
the non-chosen score is maximum when $\prefscore(C',C) = 1$ for each $C' \in \exsibcon(\dialogue,2i+1)$ and each $C \in \concerns(\dialogue,2i+1)$. 
In which case, the non-chosen score can be rewritten as $1/x(xy(1/y))$,
which is 1.
The non-chosen score is minimum when $\prefscore(C',C) = 0$ for each $C' \in \exsibcon(\dialogue,2i+1)$ and each $C \in \concerns(\dialogue,2i+1)$, 
or when $\exsibcon(\dialogue,2i+1) = \emptyset$.
In which case, the non-chosen score is 0.
\end{proof}

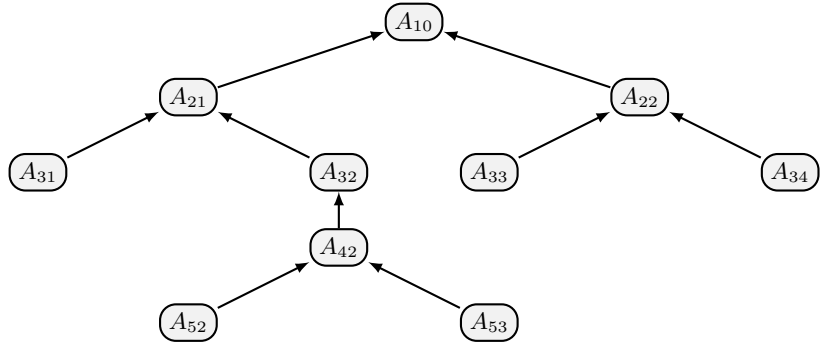
\begin{figure}
\begin{center}
\begin{tikzpicture}[->,>=latex,thick]
\tikzstyle{every node}=[{draw,text centered, shape=rectangle, rounded corners=6pt, fill=gray!10,font=\small}] 
\node (a10) [] at (5,3) {$A_{10}$};
\node (a21) [] at (2,2){$A_{21}$};
\node (a22) [] at (8,2) {$A_{22}$};
\node (a31) [] at (0,1){$A_{31}$};
\node (a32) [] at (4,1) {$A_{32}$};
\node (a33) [] at (6,1){$A_{33}$};
\node (a34) [] at (10,1){$A_{34}$};
\node (a42) [] at (4,0) {$A_{42}$};
\node (a52) [] at (2,-1) {$A_{52}$};
\node (a53) [] at (6,-1) {$A_{53}$};
\path	(a21) edge (a10);
\path	(a22) edge (a10);
\path	(a31) edge (a21);
\path	(a32) edge (a21);
\path	(a33) edge (a22);
\path	(a34) edge (a22);
\path	(a42) edge (a32);
\path	(a52) edge (a42);
\path	(a53) edge (a42);
\end{tikzpicture}
\end{center}
\caption{\label{fig:concernscore}Argument graph used in Example \ref{ex:concernscore}}
\end{figure}

\begin{example}
\label{ex:concernscore}
Consider the following dialogue based on the argument graph in Figure \ref{fig:concernscore} where the graph and dialogue are hypothetical: 

\begin{center}
\begin{tabular}{ccc}
\toprule
Step & Agent & Arguments played\\
\midrule
1 & System & $A_{10}$\\
2 & User & $A_{21},A_{22}$\\
3 & System & $A_{32},A_{33}$\\
4 & User & $A_{42}$\\
5 & System & $A_{52}$\\
\bottomrule
\end{tabular}
\end{center}

Assume that the concerns associated with the arguments and the population preference scores are as follows:

\[
\begin{array}{l c r}
 \concerns(A_{31}) = \{C_1\} & \concerns(A_{32}) = \{C_2\} & \concerns(A_{33}) = \{C_3\} \\ 
 \concerns(A_{34}) = \{C_4\} & \concerns(A_{52}) = \{C_2\} & \concerns(A_{53}) = \{C_3\} \\
\prefscore(C_1,C_2) = 1/4 & \prefscore(C_1,C_3) = 1/4 & \prefscore(C_2,C_3) = 1/4 \\
\prefscore(C_4,C_2) = 1/4 & \prefscore(C_4,C_3) = 1/4 & \\
\end{array}
\]

%
From this, we obtain the concern score as follows:
\[\arraycolsep=1.2pt
\begin{array}{lll}
\concernscore(\dialogue) & = &\frac{1}{2} \times ((1 - \nonchosenscore(\dialogue,1)) + (1 - \nonchosenscore(\dialogue,2)))\\
& = &\frac{1}{2} \times (\frac{3}{4} + \frac{3}{4}) = \frac{3}{4}\\
\end{array}
\]
where
\[\arraycolsep=1.2pt
\begin{array}{lll}
\nonchosenscore(\dialogue,1) &=& \frac{1}{2} \times (\frac{\prefscore(C_1,C_2)}{2} 
+ \frac{\prefscore(C_4,C_2)}{2} 
+ \frac{\prefscore(C_1,C_3)}{2} 
+ \frac{\prefscore(C_4,C_3)}{2})\\
&=& \frac{1}{2} \times (\frac{1}{8} + \frac{1}{8} + \frac{1}{8} + \frac{1}{8})
= \frac{1}{4}\\
\end{array}
\]
and
\[\arraycolsep=1.2pt
\begin{array}{lll}
\nonchosenscore(\dialogue,2) 
&=& \frac{1}{1} \times (\frac{\prefscore(C_3,C_2)}{1})
= \frac{1}{1} \times \frac{1}{4}
= \frac{1}{4}.\\
\end{array}
\]
\end{example} 

The concern score combines the information about the concerns associated with the arguments appearing and not appearing in the dialogue, and the relative preference over those concerns, to a single value in the unit interval. The definition incorporates a bias favouring arguments that have fewer concerns associated with them. This and other features of the definition could be further investigated, and alternative definitions could be devised and harnessed with our framework for strategic argumentation. We leave this task for future 
work.

\paragraph{Updating the user's beliefs}\hfill
\label{subsub:updating}

The user model gives the predicted beliefs of the user in the arguments at the start of the dialogue.
By the end of the dialogue, the user's beliefs may have changed as a result of the discussion. 
We therefore need an appropriate update function for producing new, more accurate beliefs throughout the dialogue. 

In principle, we could update a user model during a dialogue using a belief redistribution function that takes the old probability distribution and returns a revised one. 
To do this, we could consider the notion of an \textbf{update method} $\sigma(\prob_{i-1},\dialogue(i)) = \prob_i$, 
generating a belief distribution $\prob_i$ from $\prob_{i-1}$ based on the move $\dialogue(i)$ in dialogue $\dialogue$.
If the move is a posit of argument $A$, the belief in $A$ and its ancestors should be modified. 

However, if the update is applied on all of the possible subsets of arguments, it may lead to a computationally intractable problem.
To address this issue, we can for instance exploit the structure of the argument graph $\cal G$ \cite{HadouxHunter16} or define the belief directly on the singleton arguments \cite{HunterT16,Potyka19a,PPH19a}.
We choose the latter in this work as it is a computationally efficient option, and it allows us to modulate the update in terms of the attackers as we describe below. Furthermore, as we will see, we will update the values after a sequence of moves have been made rather than on a step-by-step basis.

\begin{definition}
A \textbf{probability labelling} is defined as $\lab: \args(\graph) \rightarrow [0,1]$. 
The probability labelling associated with a probability distribution $\prob$ is 
$\lab_{\prob}$ s.t. $\lab_{\prob}(A) = \prob(A)$ for every $A \in \args(\graph)$.
\end{definition}

For further details on the properties of such labellings we refer to \cite{HunterT16,Potyka19a,PPH19a}. What is important to note is that every probability distribution has a corresponding labelling, and for every labelling we can find at least one probability distribution producing it.

We also need the following notion of an {\bf induced graph}, i.e. the subgraph of $\cal G$ containing exactly the arguments (and any relations between them) occurring in a dialogue up to a given step $i$:

\begin{definition}
Let $\dialogue$ be a dialogue and $\graph$ a graph. 
With $\graph_{\dialogue(i)} = (X, (X \times X) \cap \arcs(\graph))$ we denote \textbf{the graph induced by the dialogue} up to step $i$, where $X = \bigcup_{j=1}^i \args(\dialogue(i))$.
By $\attackers_i(A)$ we understand the attackers of $A$ in ${\cal G}_{D(i)}$.
\end{definition}
 
\begin{example}
Consider the argument graph $\cal G$ in Figure \ref{fig:concernscore}.
If $D(1) = \{A_{10}\}$, 
$D(2) = \{A_{21},A_{22}\}$,
and 
$D(3) = \{A_{32}\}$,
then the graph induced up to step $3$ is visible in Figure \ref{fig:indgraph}.
Using the induced graph, we observe that $\attackers_3(A_{10}) = \{A_{21},A_{22}\}$,
$\attackers_3(A_{21}) = \{A_{32}\}$,
$\attackers_3(A_{22}) = \emptyset$,
and
$\attackers_3(A_{32}) = \emptyset$.
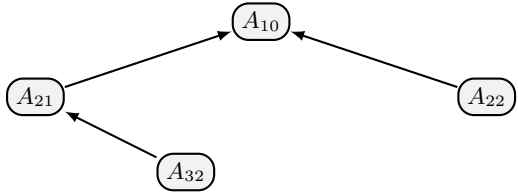
\begin{figure}[h]
\centering
\begin{tikzpicture}[->,>=latex,thick]
\tikzstyle{every node}=[{draw,text centered, shape=rectangle, rounded corners=6pt, fill=gray!10,font=\small}] 
\node (a10) [] at (5,3) {$A_{10}$};
\node (a21) [] at (2,2){$A_{21}$};
\node (a22) [] at (8,2) {$A_{22}$};
\node (a32) [] at (4,1) {$A_{32}$};
\path	(a21) edge (a10);
\path	(a22) edge (a10);
\path	(a32) edge (a21);
\end{tikzpicture}
\caption{Example of an induced argument graph.}
\label{fig:indgraph}
\end{figure}
\end{example} 

Throughout the rest of this section, we will assume that we are working with a graph $G_{D(n)}$ induced by a dialogue $(D(1), \ldots, D(n))$.

We consider three stages of belief for an argument (and hence three labellings) as arising in a dialogue (and so attack and defence is with respect to the arguments in the dialogue).

\begin{itemize}
\item the initial belief \textbf{(init)} when the argument has just been played, 
\item the attacked belief \textbf{(att)} after the argument has been attacked, and 
\item the reinstated belief \textbf{(reinst)} if the argument is  defended at the end of the dialogue. 
\end{itemize}
The reinstated belief corresponds to the value we use to evaluate the belief in an argument (and thus in the goal) at the end of the dialogue. Note, when a belief is reinstated, it is not necessarily reinstated to its original value. Rather, depending on the belief in its attackers, the reinstated belief may be below its original value, as shown in \cite{Rahwan2011}.

    
    
    
        
        
    

In order to model this behaviour of partial effectiveness of updating belief, we introduce the following coefficients to play 
the role of dampening factors (as suggested in \cite{Bonzon2017}) which causes the effect of an attacker to decrease as the length of the chain of arguments increases. We will introduce the  $k^A_{inter}$ coefficient that we will later use to obtain the attacked belief from the initial belief for a given argument $A$, and the $k^A_{reinst}$ coefficient that we will later use to obtain the reinstated belief from the attacked belief. 



\begin{definition}
For $\sigma \in \{inter,reinst\}$,
the  {\bf $\sigma$ coefficient}, denoted $k^A_{\sigma}$, is defined as follows: 
If $\attackers_n(A)$ is empty, 
then $k^A_{\sigma} = 1$, else it is defined 
as follows where $init(B)$ (respectively $reinst(B)$) is the initial (respectively reinstated) belief of the attacker $B$ .
\[
k^A_{\sigma} = \sum_{X \subseteq \attackers_n(A)} (-1)^{|X|} \times \prod_{B \in X} \sigma(B)
\]
\end{definition}




The definition of these coefficients provides a balance between how strongly the attackers are believed and their number. There are other ways to aggregate beliefs of attackers, each with their advantages and disadvantages. 
Summing the attackers' beliefs, while simple, may not be in [0,1], whereas taking the maximal or average belief does not 
reflect the number of attackers.
Our definition addresses these issues and additionally provides a form of dampening effect so that the influence of an argument decreases as the length of the chain of arguments increases. 


In the following, we show 
that: the coefficients are in the unit interval (Proposition \ref{prop:kunit});
when there are attackers of an argument,
and all the attackers have an initial belief of 1 (respectively 0), then the $k^A_{inter}$ (respectively $k^A_{reinst}$) coefficient is 0 (respectively 1) 
(Proposition \ref{prop:kextreme});
increasing the initial belief in the attackers decreases the initial belief in the attackee (Proposition \ref{prop:kstronger}):
and increasing the set of attackers decreases the initial belief in the attackee 
(Proposition \ref{prop:kmore}).

\begin{proposition}
\label{prop:kunit}
For an argument $A$, $k^A_{init} \in [0,1]$ 
and $k^A_{reinst} \in [0,1]$
\end{proposition}

\begin{proof}
Let $n$ denote the cardinality of $\attackers_n(A)$.
First, assume that the init value for each argument in $\attackers_n(A)$ is the value $b$.
We can rewrite $k_{init}$ as $\sum_{i=o}^n t_i$
where $t_i = (-1)^i \times \binom{i}{n} \times b^i$.
So $k^A_{init} = (1-b)^n$.
Hence, $k^A_{init} \in [0,1]$.
We can generalize to handle attacks with different beliefs.
For each $i$, let $b_i$ be the belief for argument $A_i$.
So we can rewrite $k_{init}$ 
as $(1-b_1)(1-b_2)\dots(1-b_n)$.
Since $b_i \in [0,1]$ for all $i$ (recall that $b_i$ denotes belief), $k^A_{init} \in [0,1]$.
\end{proof}

\begin{proposition}
\label{prop:kextreme}
If $\attackers_n(A) \neq \emptyset$
and for all $B \in \attackers_n(A)$,
$init(B) = 1$ (respectively $init(B) = 0$),
then $k^A_{init} = 0$ (respectively $init(B) = 1$).
\end{proposition}

\begin{proof}
Assume $\attackers_n(A) \neq \emptyset$.
Let $n$ denote the cardinality of $\attackers_n(A)$.
We can rewrite $k^A_{init}$ as $\sum_{i=0}^n t_i$
where $t_i = 1$ and for all $i > 1$, $t_i = (-1)^i \times \binom{i}{n} \times b^i$.
First, assume for all $B \in \attackers_n(A)$, $init(B) = 1$.
So $k^A_{init} = \sum_{i=0}^n (-1)^i \times \binom{i}{n} = 1$
Now, assume for all $B \in \attackers_n(A)$, $init(B) = 0$.
So for all $i > 1$, $t_i = 0$.
Hence, $k^A_{init} = 1$.
\end{proof}

In the following result, we assume that the attackers for $A$ are given in some arbitrary ordering (i.e. the ordering has no meaning), and then we can put the attackers in $A'$ in an order so that for each $i$, $init(B_i) < init(B'_i)$.


\begin{proposition}
\label{prop:kstronger}
For arguments $A$ and $A'$, 
if it possible to order $\attackers_n(A)$ as the sequence $B_1,\ldots,B_m$
and $\attackers_n(A')$ as the sequence $B'_1,\ldots,B'_m$
such that
for each $i$, $init(B_i) < init(B'_i)$,
then $k^A_{init} > k^{A'}_{init}$.
\end{proposition}

\begin{proof}
From the proof of Proposition \ref{prop:kunit},
we have $k^A_{init} = \prod^n_{i=1} (1 - b_i)$
and $k'^A_{init} = \prod^n_{i=1} (1 - b'_i)$.
Since, for each $i$, $b_i < b'_i$,
we have $k^A_{init} > k^{A'}_{init}$. 
\end{proof}

Like in the previous result, the following result assumes that the attackers for $A$ are given in some arbitrary ordering, and then we can put the first $m$ attackers in $A'$ in an order so that for each $i$, $init(B_i) < init(B'_i)$.

\begin{proposition}
\label{prop:kmore}
For arguments $A$ and $A'$, 
if it possible to order $\attackers_n(A)$ as the sequence $B_1,\ldots,B_m$
and $\attackers_n(A')$ as the sequence $B'_1,\ldots,B'_n$
such that
for each $i \leq m$, $init(B_i) = init(B'_i)$,
and $m < n$, 
then $k^A_{init} < k^{A'}_{init}$.
\end{proposition}

\begin{proof}
From the proof of Proposition \ref{prop:kunit},
we have $k^A_{init} = \prod^n_{i=1} (1 - b_i)$.
So $k^{A'}_{init} = k^A_{init} \times (1 - b_{m+1}) \times \ldots \times (1 -b_n)$.
Hence, $k^{A'}_{init} > k^A_{init}$.
\end{proof}

Now we define the way initial, attacked, and reinstated belief are calculated for each argument $A$. In this paper, we assume that for the $init$ labeling, $init(A)$ is obtained directly by sampling the beta distribution for $A$. In other words, $init(A)$ takes a particular value in the unit value with a probability given by the beta distribution as explained in Section \ref{sub:beliefs}. So for instance, if the beta distribution is a bell shape, then the $init(A)$ will be the value at the peak with highest probability, and the further $init(A)$ is away from the peak, the lower the probability of this occurring. 


\begin{definition}
\label{def:propagation}
Let $att$ and $reinst$ be probability labelings.
For argument $A$, the $att(A)$ and $reinst(A)$ values are calculated as follows.
\[
att(A) = 
\begin{cases}
init(A) \times k^A_{init} 
& \mbox{ if there is a } B \in \attackers_n(A) \mbox{ s.t. } init(B) > 0.5\\
init(A) 
& \mbox{ otherwise}
\end{cases}
\]
\[
reinst(A) = 
\begin{cases}
att(A) + (k^A_{reinst} \times (1 - att(A))
& \mbox{if } \attackers_n(A) \neq \emptyset \\
& \quad \mbox{and }  \mbox{for all } B\in \attackers_n(A),\\
& \quad reinst(B) \leq 0.5
\\
att(A) 
& \mbox{otherwise}
\end{cases}
\]
\end{definition}

So $att(A)$ is $init(A)$ decreased by multiplication with the $k^A_{init}$ coefficient when $A$ has an attacker; otherwise, it is unchanged. Hence, $att(A)$ is a local calculation that only takes into account the initial belief in the immediate attackers and does not take into account attackers of attackers, and so on by recursion.

In contrast, $reinst(A)$ is $att(A)$ increased by adding $k^A_{reinst} \times (1 - att(A))$.
Therefore,  the calculation of $reinst(A)$ takes into account the reinstated value of the attackers, and this in turn takes into account the reinstated value of its attackers, and so on.
This method is based on the proposal for ambiguous updates for probabilistic argumentation \cite{Hunter2015ijcai}, 
and we use its equivalent version that is defined directly in terms of belief in arguments \cite{HadouxHunter2018foiks}.


\begin{figure}
\centering
\begin{tikzpicture}[->,>=latex,thick]
\tikzstyle{every node}=[draw,rectangle,rounded corners=6pt,fill=gray!20]
\node (a1) [] at (0,0) {\footnotesize $A_1$};
\node (a2) [] at (2,0){\footnotesize $A_2$};
\node (a3) [] at (4,0) {\footnotesize $A_3$};
\node (a4) [] at (6,0) {\footnotesize $A_4$};
\path	(a2) edge (a1);
\path	(a3) edge (a2);
\path	(a4) edge (a3);
\end{tikzpicture}
\caption{\label{fig:beliefupdate}Argument graph used in Example \ref{ex:beliefupdate}}
\end{figure}
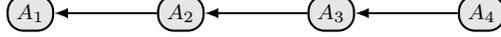

\begin{example}
\label{ex:beliefupdate}
Consider the argument graph in Figure \ref{fig:beliefupdate}.
The following table gives the $init$, and the $att$ and $reinst$ values that are calculated from them. 
\begin{center}
\begin{tabular}{lrrrr}
\toprule
& $A_1$ & $A_2$ & $A_3$ & $A_4$\\
\midrule
$init$ & $0.8$ & $0.6$ & $0.7$ & $0.2$ \\
$k^A_{init}$ & $0.4$ & $0.3$ & $1$ & $1$\\
$att$ & $0.32$ & $0.18$ & $0.7$ & $0.2$ \\
$k^A_{reinst}$ & $0.574$ & $0.3$ & $1$ & $1$ \\
$reinst$ & $0.71$ & $0.426$ & $0.7$ & $0.2$\\
\bottomrule
\end{tabular}\end{center}
Here we see that because $A_4$ is not believed, it has no effect on the calculation of the updated values, whereas $A_2$ is able to attack $A_1$ and thereby decrease the $att$ value of $A_1$, and $A_3$ is able to attack $A_2$ and thereby decrease the $att$ value of $A_2$. Furthermore, the attack by $A_3$ on $A_2$ prevails and so the reinst value of $A_2$ is below $0.5$, whereas the attack by $A_2$ on $A_1$ does not prevail, and so the reinst value of $A_1$ rises above 0.5 but not back to its initial value. 
\end{example}

We use the formulae in Definition \ref{def:propagation} for several reasons.
First, the definition takes the initial belief into account (where that comes from the beta distribution).
Second, the definition takes the number of attackers into account. 
In a fine-grained setting, it is normally difficult to have a uniform intuition as to whether a single attacker with high belief should be more effective in decreasing the belief on the attackee than 10 attackers with lower beliefs. 
Using our $k$ calculation, we take into account both the number of attackers and their respective weight in the set of attackers.
Finally, we have made the common assumption that the reinstated belief should be lower than the initial belief in the argument 
(see \cite{Rahwan2011} for an empirical study and \cite{PolbergHunter2018ijar} for a discussion).
For this, the use  of the coefficients has a dampening effect so that an indirect attacker or defender has decreasing effect as the length of the path to the attacked or defended argument increases.


  
The $reinst$ method is one of a range of possibilities for defining update methods, and we use it as an example of an update method for our framework. However, it would be straightforward to adopt an alternative method such as suggested in \cite{Hunter2015ijcai,Hunter2016sum} in our framework. 
Possibly, as an alternative, we could harness recent developments in weighted and ranking semantics (for example \cite{Amgoud2013,Amgoud2016,Bonzon2017,Baroni2019}), though this would take us away from a probabilistic interpretation of belief in arguments, which would in turn create challenges in how to acquire and interpret user weighting of arguments.

\paragraph{Combining concerns and beliefs}
\label{sec:rerwardfunction}
We now combine the two dimensions of the reward that we have specified in Sections \ref{section:scoringconcerns} and \ref{subsub:updating} as follows.

\begin{definition}
For a dialogue $\dialogue$, and a persuasion goal $A$,
the {\bf reward function} is 
\[
\reward(\dialogue) = \concernscore(\dialogue) \times reinst(A)
\]
\end{definition}

Put simply, the reward function is the product of the two dimensions. This means we give equal weight to the two dimensions. It also means that weakness in either dimension is sufficient to give a low reward. 

\begin{example}
We continue Example \ref{ex:concernscore} where for the dialogue $\dialogue$, $\concernscore(\dialogue) = 0.75$.
Suppose for the persuasion goal $A_{10}$, we have $reinst(A_{10}) = 0.8$. 
So $\reward(\dialogue) = 0.75 \times 0.8 = 0.6$
\end{example}

The reward function is a simple and intuitive way of aggregating the two dimensions, but other aggregations could be specified and used directly in our framework.


\subsubsection{Simulating a choice from the user}


When simulating the results of the system's actions, we also need to simulate a credible behaviour from the user in order 
to advance in the simulated dialogue. Thus, it is important to mimic the choices of arguments that the user could make.  
In order to do that, we propose a multistep process, for each argument $A$ from the system to counter:

\begin{enumerate}

    \item \textbf{Sample}: for each counterarguments $B$ to $A$, we sample from the beta mixture to simulate whether the user believes $B$ or not, \emph{i.e.}, whether the value drawn from the mixture is greater than $0.5$ or not;
    
    \item \textbf{Order}: the set of believed counterarguments $bc$ is then totally ordered w.r.t. $$\mathrm{score} = P(A) \times \left(\sum_{C \in {\sf Con}(A)} \sum_{C' \in {\sf SibCon}(D, A)} \mathds{1}_{C' \conpref_{\concerndomain}^{\dusertag} C} \right) \times \frac{1}{| {\sf SibCon}(D, A)|}, \forall A \in bc$$ 
    In other words, the arguments are ranked from the user point of view, according to her preferences and belief;
   
    \item \textbf{Filter:} to take into account the fact that a user can decide to withhold arguments she believes, we randomly draw a subset size $t \in \{0, 1 \ldots, |bc|\}$. We then only consider the $t$ first arguments in $bc$ starting from the first in the order defined at the previous step.
\end{enumerate}

All the subsets of believed counterarguments to the arguments played by the system are then used as the new step
 in the simulated  dialogue, \emph{i.e.}, a new state in the Monte-Carlo tree, from which any subsequent simulation 
for this line of dialogue will start.

The assumptions that we have made for the simulation of the user choices are supposed to encompass any possible actual behaviour for the users (supposing their choices are based on the same elements: belief in arguments, ranking on them and then choice of whether to play them or not).
Therefore, in theory, with enough simulation, the strategy that we obtain is the most robust way for facing the real users.

\subsection{Baseline strategy}
The baseline and advanced systems use the same protocol and the same argument graphs. 
This means that the the baseline and advanced systems only differ on the argument selection strategy.
The baseline strategy is a form of random strategy: 
When the baseline system has a choice of counterargument to present, it makes a selection  using a uniform random distribution.

\section{Data for domain and user modelling}
\label{section:datascience}

In this section, we describe the methods used for obtaining the data we required for populating our domain and user model. The domain model is based on an argument graph (Section \ref{sec:experimentgraphs}) and the assignment of concerns to arguments (Section \ref{sec:concernsarguments}), and the user model is based on the belief in arguments (Section \ref{sec:beliefarguments}), preferences over concerns (Section \ref{section:preferenceconcerns}), and classification trees for predicting preferences over concerns (Section \ref{sub:tree}). This data is available as an appendix\footnote{The Data Appendix containing the arguments, argument graphs, assignment of concerns to arguments, preferences over concerns, and assignment of beliefs to arguments, is available at the link \url{http://www0.cs.ucl.ac.uk/staff/a.hunter/papers/unistudydata.zip}}.

For each task, we have recruited a certain number of participants from a crowdsourcing platform called Prolific\footnote{\url{https://prolific.ac}} and ensured the quality of their responses by using appropriate attention checks. We have also excluded the participants of one task from taking part in another task. We note that while the participants came from a single platform, various tools were used to create the questionnaires and tasks themselves, depending on the kind of functionality that was needed. 

\subsection{Argument graphs}
\label{sec:experimentgraphs}

For the experiments that we present in Section \ref{section:experiments}, we used two argument graphs on the topic of charging students a fee for attending university. Since the experiments were conducted with participants from the UK, the arguments used in the argument graph pertain to the UK context. In the UK, the current situation is that students from the UK or other EU countries pay \pounds 9K per year for tuition at most universities. This is a controversial situation, with some people arguing for the fees to be abolished, and with others arguing that they should remain in place.  

Each argument graph has a persuasion goal. For the first argument graph, the persuasion goal is {\em \enquote{Charging students the \pounds 9K fee for university education should be abolished}}, and for the second argument graph, the persuasion goal is {\em \enquote{Universities should continue charging students the \pounds 9K fee}}. The reason we have two graphs is that when we ran the experiments, we asked a participant whether they believe the \pounds 9K tuition fees should be abolished or maintained. If they believed that the fees should be maintained, then the APSs used the first argument graph (\emph{i.e.} the graph in favour of abolishing the fees), to enter the discussion, whereas if they believed that they should be abolished, then the APSs used the second argument graph (\emph{i.e.} the graph in favour of continuing charging the fees). 

The arguments were hand-crafted, however, the information in them has been obtained 
from diverse sources such as online forums and newspaper articles so that it would reflect a diverse range of opinions. 
These arguments are enthymemes (\emph{i.e.}, some premises/claims are implicit) as this offers more natural exchanges in the dialogues. 
We obtained 146 arguments, which were then used to construct the two argument graphs (the first graph contains 106 of them, while the second 119). Hence, many of the arguments are shared, but often play contrary roles (\emph{i.e.} a defender of the persuasion goal in one graph was typically an attacker in the other).
In the context of this work, we only deal with the attack relation, and so we did not consider other kinds of interactions such as support. Furthermore, we did not attempt to distinguish between the different kinds of attack (such as undercutting or undermining). Some arguments were edited to enable us to have reasonable depth (so that the dialogues were of a reasonable length) and breadth (so that alternative dialogues were possible) to the argument graph. 
The full list of the arguments is available in the Data Appendix, and the structure for one of the argument graphs is also presented in Figure \ref{fig:unistudygraph}.


For the following data gathering steps, we split the 146 arguments into 13 groups (the groups are distinguished in the Data Appendix files associated with surveys in which grouping was needed).  
We did this so that no group contained two directly related arguments (\emph{i.e.} no arguments where one argument attacks the other) and almost all groups do not contain sibling arguments (\emph{i.e.} attacking the same argument). The aim of this was to avoid the participants consciously or subconsciously evaluating interactions between arguments when undertaking the tasks in the following steps. In other words, when a participant was given a group of arguments, we wanted them to consider the arguments individually and not collectively. 

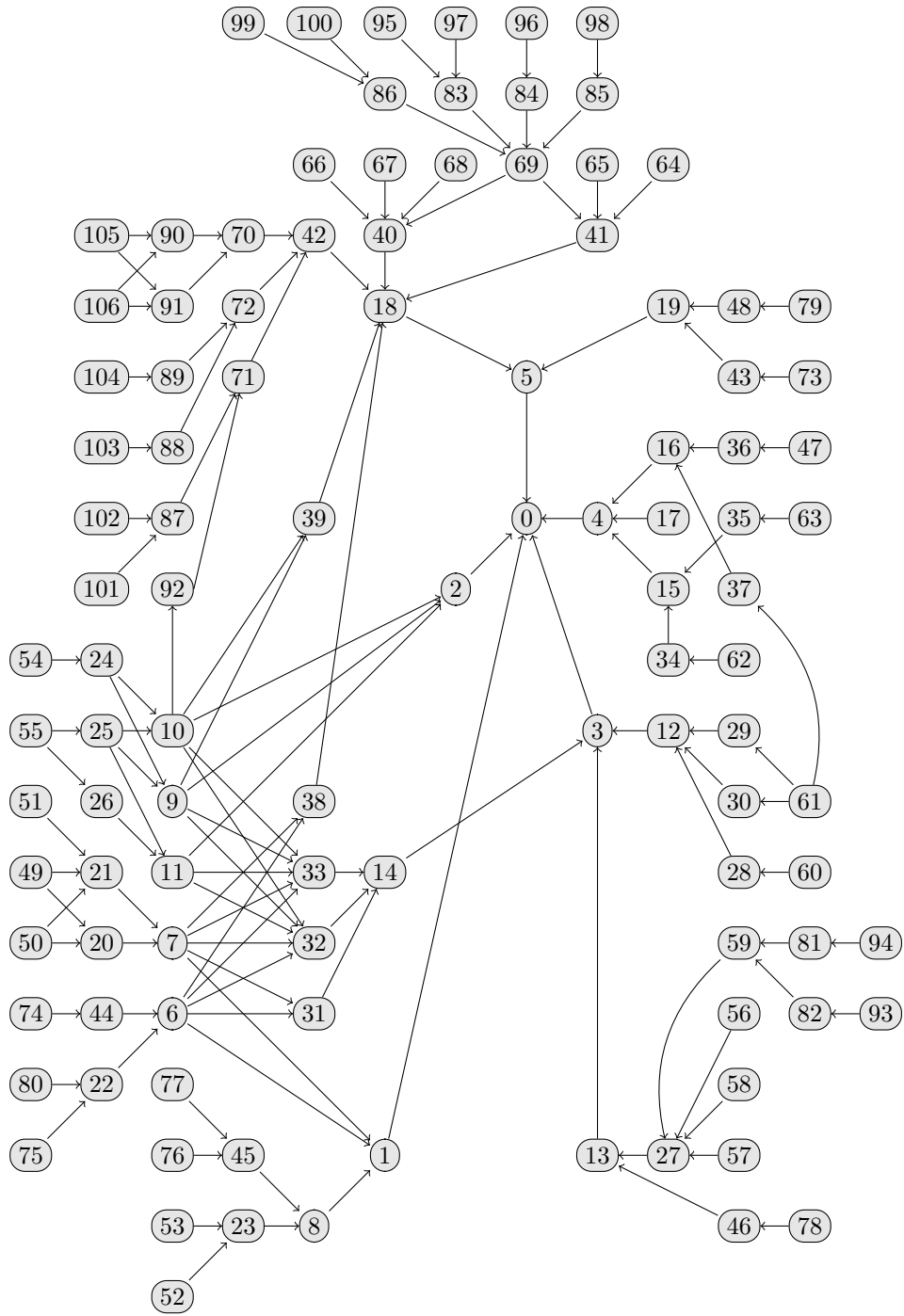
\begin{figure}
\centering

\begin{tikzpicture}[->,every node/.style={draw,rectangle,rounded corners=6pt,fill=gray!20}]


\node (99) [] at (5,18) {99};
\node (100) [] at (6,18) {100};
\node (95) [] at (7,18) {95};
\node (97) [] at (8,18) {97};
\node (96) [] at (9,18) {96};
\node (98) [] at (10,18) {98};

\node (86) [] at (7,17) {86};
\node (83) [] at (8,17) {83};
\node (84) [] at (9,17) {84};
\node (85) [] at (10,17) {85};

\node (66) [] at (6,16) {66};
\node (67) [] at (7,16) {67};
\node (68) [] at (8,16) {68};
\node (69) [] at (9,16) {69};
\node (65) [] at (10,16) {65};
\node (64) [] at (11,16) {64};

\node (105) [] at (3,15) {105};
\node (90) [] at (4,15) {90};
\node (70) [] at (5,15) {70};
\node (42) [] at (6,15) {42};
\node (40) [] at (7,15) {40};
\node (41) [] at (10,15) {41};

\node (106) [] at (3,14) {106};
\node (91) [] at (4,14) {91};
\node (72) [] at (5,14) {72};
\node (18) [] at (7,14) {18};
\node (19) [] at (11,14) {19};
\node (48) [] at (12,14) {48};
\node (79) [] at (13,14) {79};

\node (104) [] at (3,13) {104};
\node (89) [] at (4,13) {89};
\node (71) [] at (5,13) {71};
\node (5) [] at (9,13) {5};
\node (43) [] at (12,13) {43};
\node (73) [] at (13,13) {73};

\node (103) [] at (3,12) {103};
\node (88) [] at (4,12) {88};
\node (16) [] at (11,12) {16};
\node (36) [] at (12,12) {36};
\node (47) [] at (13,12) {47};

\node (102) [] at (3,11) {102};
\node (87) [] at (4,11) {87};
\node (39) [] at (6,11) {39};
\node (0) [] at (9,11) {0};
\node (4) [] at (10,11) {4};
\node (17) [] at (11,11) {17};
\node (35) [] at (12,11) {35};
\node (63) [] at (13,11) {63};

\node (101) [] at (3,10) {101};
\node (92) [] at (4,10) {92};
\node (2) [] at (8,10) {2};
\node (15) [] at (11,10) {15};
\node (37) [] at (12,10) {37};

\node (54) [] at (2,9) {54};
\node (24) [] at (3,9) {24};
\node (34) [] at (11,9) {34};
\node (62) [] at (12,9) {62};

\node (55) [] at (2,8) {55};
\node (25) [] at (3,8) {25};
\node (10) [] at (4,8) {10};
\node (3) [] at (10,8) {3};
\node (12) [] at (11,8) {12};
\node (29) [] at (12,8) {29};

\node (51) [] at (2,7) {51};
\node (26) [] at (3,7) {26};
\node (9) [] at (4,7) {9};
\node (38) [] at (6,7) {38};
\node (30) [] at (12,7) {30};
\node (61) [] at (13,7) {61};

\node (49) [] at (2,6) {49};
\node (21) [] at (3,6) {21};
\node (11) [] at (4,6) {11};
\node (33) [] at (6,6) {33};
\node (14) [] at (7,6) {14};
\node (28) [] at (12,6) {28};
\node (60) [] at (13,6) {60};

\node (50) [] at (2,5) {50};
\node (20) [] at (3,5) {20};
\node (7) [] at (4,5) {7};
\node (32) [] at (6,5) {32};
\node (59) [] at (12,5) {59};
\node (81) [] at (13,5) {81};
\node (94) [] at (14,5) {94};

\node (74) [] at (2,4) {74};
\node (44) [] at (3,4) {44};
\node (6) [] at (4,4) {6};
\node (31) [] at (6,4) {31};
\node (56) [] at (12,4) {56};
\node (82) [] at (13,4) {82};
\node (93) [] at (14,4) {93};

\node (80) [] at (2,3) {80};
\node (22) [] at (3,3) {22};
\node (77) [] at (4,3) {77};
\node (58) [] at (12,3) {58};

\node (75) [] at (2,2) {75};
\node (76) [] at (4,2) {76};
\node (45) [] at (5,2) {45};
\node (1) [] at (7,2) {1};
\node (13) [] at (10,2) {13};
\node (27) [] at (11,2) {27};
\node (57) [] at (12,2) {57};

\node (53) [] at (4,1) {53};
\node (23) [] at (5,1) {23};
\node (8) [] at (6,1) {8};
\node (46) [] at (12,1) {46};
\node (78) [] at (13,1) {78};
\node (52) [] at (4,0) {52};

\path	(1) edge[] (0);
\path	(2) edge[] (0);
\path	(3) edge[] (0);
\path	(4) edge[] (0);
\path	(5) edge[] (0);
\path	(6) edge[] (1);
\path	(6) edge[] (31);
\path	(6) edge[] (38);
\path	(6) edge[] (32);
\path	(6) edge[] (33);
\path	(7) edge[] (1);
\path	(7) edge[] (31);
\path	(7) edge[] (38);
\path	(7) edge[] (32);
\path	(7) edge[] (33);
\path	(8) edge[] (1);
\path	(9) edge[] (32);
\path	(9) edge[] (33);
\path	(9) edge[] (2);
\path	(9) edge[] (39);

\path	(10) edge[] (32);
\path	(10) edge[] (33);
\path	(10) edge[] (2);
\path	(10) edge[] (39);
\path	(10) edge[] (92);
\path	(11) edge[] (32);
\path	(11) edge[] (33);
\path	(11) edge[] (2);
\path	(12) edge[] (3);
\path	(13) edge[] (3);
\path	(14) edge[] (3);
\path	(15) edge[] (4);
\path	(16) edge[] (4);
\path	(17) edge[] (4);
\path	(18) edge[] (5);
\path	(19) edge[] (5);

\path	(20) edge[] (7);
\path	(21) edge[] (7);
\path	(22) edge[] (6);
\path	(24) edge[] (9);
\path	(24) edge[] (10);
\path	(25) edge[] (11);
\path	(25) edge[] (9);
\path	(25) edge[] (10);
\path	(26) edge[] (11);
\path	(23) edge[] (8);
\path	(27) edge[] (13);
\path	(28) edge[] (12);
\path	(29) edge[] (12);

\path	(30) edge[] (12);
\path	(31) edge[] (14);
\path	(32) edge[] (14);
\path	(33) edge[] (14);
\path	(34) edge[] (15);
\path	(35) edge[] (15);
\path	(36) edge[] (16);
\path	(37) edge[] (16);
\path	(38) edge[] (18);
\path	(39) edge[] (18);

\path	(40) edge[] (18);
\path	(41) edge[] (18);
\path	(42) edge[] (18);
\path	(43) edge[] (19);
\path	(44) edge[] (6);
\path	(45) edge[] (8);
\path	(46) edge[] (13);
\path	(47) edge[] (36);
\path	(48) edge[] (19);
\path	(49) edge[] (20);
\path	(49) edge[] (21);

\path	(50) edge[] (20);
\path	(50) edge[] (21);
\path	(51) edge[] (21);
\path	(52) edge[] (23);
\path	(53) edge[] (23);
\path	(54) edge[] (24);
\path	(55) edge[] (25);
\path	(55) edge[] (26);
\path	(56) edge[] (27);
\path	(57) edge[] (27);
\path	(58) edge[] (27);
\path	(59) edge[bend right] (27);

\path	(60) edge[] (28);
\path	(61) edge[] (30);
\path	(61) edge[] (29);
\path	(61) edge[bend right] (37);
\path	(62) edge[] (34);
\path	(63) edge[] (35);
\path	(64) edge[] (41);
\path	(65) edge[] (41);
\path	(66) edge[] (40);
\path	(67) edge[] (40);
\path	(68) edge[] (40);
\path	(69) edge[] (40);
\path	(69) edge[] (41);

\path	(70) edge[] (42);
\path	(71) edge[] (42);
\path	(72) edge[] (42);
\path	(73) edge[] (43);
\path	(74) edge[] (44);
\path	(75) edge[] (22);
\path	(76) edge[] (45);
\path	(77) edge[] (45);
\path	(78) edge[] (46);
\path	(79) edge[] (48);

\path	(80) edge[] (22);
\path	(81) edge[] (59);
\path	(82) edge[] (59);
\path	(83) edge[] (69);
\path	(84) edge[] (69);
\path	(85) edge[] (69);
\path	(86) edge[] (69);
\path	(87) edge[] (71);
\path	(88) edge[] (72);
\path	(89) edge[] (72);

\path	(91) edge[] (70);
\path	(90) edge[] (70);
\path	(92.east) edge[] (71);
\path	(93) edge[] (82);
\path	(94) edge[] (81);
\path	(95) edge[] (83);
\path	(96) edge[] (84);
\path	(97) edge[] (83);
\path	(98) edge[] (85);
\path	(99) edge[] (86);

\path	(100) edge[] (86);
\path	(101) edge[] (87);
\path	(102) edge[] (87);
\path	(103) edge[] (88);
\path	(104) edge[] (89);
\path	(105) edge[] (91);
\path	(105) edge[] (90);
\path	(106) edge[] (91);
\path	(106) edge[] (90);

\end{tikzpicture}

\caption{One of the argument graphs for the university case study. In order to better show the structure of the graph, the textual content of the arguments has been removed. The text is available in the data appendix.}
\label{fig:unistudygraph}
\end{figure}

\subsection{Belief in arguments}
\label{sec:beliefarguments}

In order to determine the belief that participants have in each argument, we used the 13 groups of arguments as described in the previous section. 
For each group, we recruited 80 participants from the Prolific crowdsourcing platform and asked each of them to assign a belief value to every argument in the group\footnote{Please note that the total number of participants was greater, but many of the submissions were rejected due to failed attention checks, which are a standard tool for rejecting dishonest participants.}. 
For each argument in the group, we asked the participants to state how much they agree or disagree both with the information and the reasoning presented in it (if applicable). For example, given a text \enquote{X therefore Y}, we asked them to consider whether they agree with X and Y and believe that X justifies Y, and to make their final judgment by looking at all of these elements.
 

For each statement, the participants could provide their belief using a slider bar that has a range from -5 to 5 and 
with a granularity of .01. This means that a participant can give a belief such as -2.89 or 0.08. We also associated a text description with each integer value as follows: (-5) {\em Strongly disagree}; (0) {\em Neither agree nor disagree}; and (5) {\em Strongly agree}.

Once all the beliefs were gathered, we calculated the beta mixture for every argument (recall Section \ref{sub:beliefs}), using the method described in \cite{HadouxHunter2018aamas}.
Using an \emph{Expectation Maximisation} (EM) algorithm, we learnt the set of components (the beta distributions) describing data the best, while taking into account the complexity of the model in order to avoid overfitting.


\begin{table}[!ht]
\begin{tabularx}{\textwidth}{l X}
\toprule
Concern & Description of what concern deals with\\
\midrule
Economy 
& Economy of the country, including public sector, private sector, import, export, taxation of companies, etc.\\
Government Finances 
& Government finances, including general taxation, government spending etc.\\
Employment 
& Careers and employability of students and the general job market.\\
Student Finances 
& Finances of students, including tuition fees, student debts, credit scores, life costs etc.\\
Education 
& Education, including the quality and value of education, grade inflation, personal development etc.\\
Student Satisfaction 
& Whether students are satisfied with their courses and universities and whether their requests are heard and met.\\
Student Well-Being 
& The physical and mental health of students, recreation and leisure activities, stress, future fears etc.\\
University Management 
& How universities are run, including university finances, competition between universities, investment into facilities or research etc.\\
Commercialization of Universities 
& How universities are commercialized, including private sector universities, treating students as customers, market forces affecting the running of universities etc.\\
Fairness 
& Whether something is fair or not (using a general understanding of fairness), including equal and just treatment of individuals.\\
Society 
& Various groups of society as well as society as a whole, and includes social mobility, minorities, disadvantaged groups of society etc.\\     
\bottomrule
\end{tabularx}
\caption{Types of concern for the topic of charging university tuition fees.}
\label{table:concerns}
\end{table}

\subsection{Concerns of arguments}
\label{sec:concernsarguments}

Once all the arguments have been defined, they need to be appropriately tagged with concerns. 
The types of concern that can be associated with the arguments are topic dependent and in this work we manually defined a set of 12 
classes (as presented in Table \ref{table:concerns}). These were based on a consideration of the different possible stakeholders who 
might have a view on university tuition fees in the UK, and what their concerns might be.  





In order to determine the concerns that the participants associate with each argument, we used the 13 groups of arguments as described previously. 
For each argument described in the Data Appendix, we asked the participants to choose the type of concern they think is the most appropriate from the list presented in Table \ref{table:concerns} (\emph{i.e.}, assign concerns to each argument that in their opinion arose or are addressed by the argument). 
The participants were restricted to assigning between 1 and 3 concerns per argument. 

The concern assignment was later post-processed in order to reduce possible noise in the data. The concerns of a given argument are ordered based on the number of times they were selected by the participants. The threshold is set at half of the number of votes of the most popular concern and only concerns above this threshold were kept. 
For instance, if \emph{Employment} is the most popular concern assigned to argument $A$ and was voted 20 times, then only concerns that have been selected by strictly more than 10 participants are assigned to this argument. The processed concern assignment can be found in the Data Appendix.

For this step, we recruited at least 40 participants from the Prolific\footnote{\url{https://prolific.ac}} crowdsourcing platform for each group of arguments \footnote{Please note that the total number of participants was greater, but many of the submissions were rejected due to failed attention checks. Due to minor platform issues, certain arguments received between 41 and 43 responses rather than 40.}.
The prescreening required their nationality to be British and age between 18 and 100. 
Only the participants who passed the prescreening were able to take part in the studies described here.


\subsection{Preferences over concerns}
\label{section:preferenceconcerns}

After the set of types of concern had been created, the next step was to determine the preferences that the users of our system could have on these types.
Preference elicitation and preference aggregation are research domains by themselves and it would take more than a paper to fully investigate them all in our context. 
Consequently, in this work, we decided to use a simple approach which was to ask the participants to provide a linear ordering over the types of concern.

It is interesting to note that the results show that on average, the \enquote{Education} and \enquote{Student Well-being} concerns were ranked respectively first and second and \enquote{Government Finances} was ranked last. 
 


For this step, we used 110 participants from the Prolific crowdsourcing website\footnote{Similarly as in the previous tasks, we used an attention checks to discard dishonest submissions.}.
The prescreening required their nationality to be British and age between 18 and 100. 
In addition to the preference task, the participants were also asked a series of profiling questions, which will be discussed in more detail in the next subsection.
The results can be found in the Data Appendix.

\subsection{Creation of classification trees}
\label{sub:tree}




Preferences are valuable tools in argumentation in general and they allow APSs to offer a more user-tailored experience. Agents may differ in their preferences and we need to discover these during a dialogue.
However, sourcing them has certain challenges. A simple way to achieve it would be to query the user. However, in practice, we do not want to ask the user too many questions as it is likely to increase the risk of them disengaging. Longer discussions also tend to be 
less effective \cite{TanNDNML16}.
Furthermore, it is normally not necessary to know about all of the preferences of the user. To address this, we can acquire comprehensive data on the preferences of a set of participants, and then use this data to train a classifier to predict the preferences for other participants.   
Thus, in this study we have created the classification trees using information that we had obtained about the users. 

In addition to asking for the ranking of the concerns of the participants (as explained in the previous subsection), we asked them to take a personality test.
We used the Ten-Item Personality Inventory (TIPI) \cite{GoslingRentfrowSwann03} to assess the values on 5 features of personality based on the OCEAN model \cite{McCraeCosta87}, one of the most famous model of the psychology literature.
These features were \enquote{Openness to experience}, \enquote{Conscientiousness}, \enquote{Extroversion}, 
\enquote{Agreeableness} and \enquote{Neuroticism} (emotional instability).
We also asked them to provide some demographic information and domain dependent information, such as age, sex, if they were a student in any higher institution, and the number of children they might have in general as well as in school or university. 

Using all the above data, we learnt a decision tree for each pair of concerns using the Scikit-learn\footnote{\url{http://scikit-learn.org}} Python library. 
The purpose of each decision tree was to determine the ratio of preference (\emph{i.e.}, for each pair of concerns, the proportion who ranked the first concern higher than the second concern) on the concerns depending on the data about the individual.
In other words, for such a decision tree, each leaf is a ratio of preference (\emph{i.e.}, the classification), and so the arcs on the branch to that leaf are for attributes that hold for the individual who is predicted to have that ratio.   

As a first stage, we ran a meta-learning process in order to determine the best combination of tree depth and minimum number of samples at each leaf for each pair of types.
The meta-learning process is the repeated application of the learning algorithm for different choices of these parameters (\emph{i.e.}, tree depth and minimum number if samples at each leaf) until the best combination of parameters is found.  
The criterion to minimize is the Hamming loss, \emph{i.e.}, the difference between the prediction and the actual preferred type.

We used cross-validation in the meta-learning to determine the best combination of tree depth and minimum number of datapoints at each leaf.
Once the best parameters were found for each pair of types, we then ran the actual learning part using these parameters with all the datapoints concerning the personality and demographic information.
We thus obtained one decision tree for each pair of types that was used by the automated persuasion system in the final study.  

Figure \ref{fig:timecomfort} shows the example of the decision tree learnt for the Economy/Fairness pair of types where ``C'' (resp. ``N'') stands for ``Conscienciousness'' (resp. ``Neuroticism'') in the OCEAN model. 

\begin{figure}[t]
    \centering
    \begin{tikzpicture}[>=stealth', thick, node distance=3cm]
      \node[state,shape=rectangle] (1) {C $< 4.75$?};
      \node[state,shape=rectangle] (2) [below left of=1] {Fairness: 77\%};
      \node[state,shape=rectangle] (3) [below right of=1] {N $< 6.25$?};
      \node[state,shape=rectangle] (4) [below left of=3] {C $< 6.25$?};
      \node[state,shape=rectangle] (5) [below left of=4] {Fairness: 57\%};
      \node[state,shape=rectangle] (6) [below right of=4] {Economy: 66\%};
      \node[state,shape=rectangle] (7) [below right of=3] {Economy: 100\%};

      \draw[-] (1) to node[left] {true} (2);
      \draw[-] (1) to node[right] {false} (3);
      \draw[-] (3) to node[left] {true} (4);
      \draw[-] (3) to node[right] {false} (7);
      \draw[-] (4) to node[left] {true} (5);
      \draw[-] (4) to node[right] {false} (6);
    \end{tikzpicture}
    \caption{Example of a decision tree for the Economy/Fairness pair where ``C'' (resp. ``N'') stands for ``Conscienciousness'' (resp. ``Neuroticism'') in the OCEAN model.}
    \label{fig:timecomfort}
\end{figure}
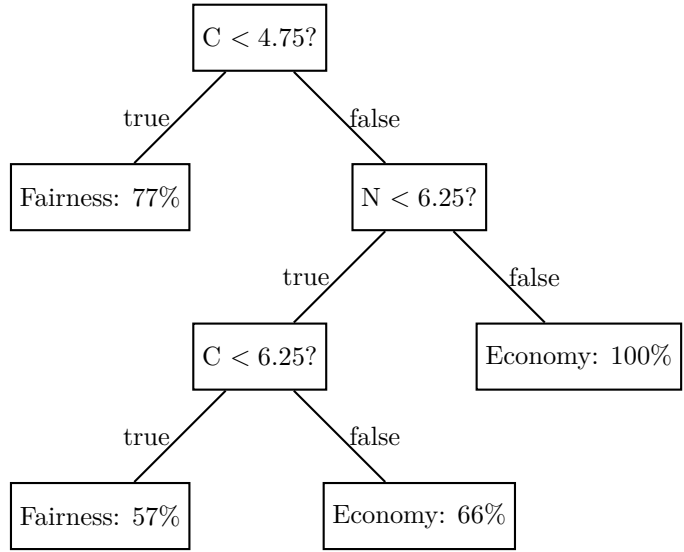

\section{Experiments}
\label{section:experiments}

We now describe the experiments we undertook to evaluate our approach to strategic argumentation for persuasion using the data and models explained in the previous parts of this paper. 


\subsection{Methods}

For the experiments, we implemented two versions of our automated persuasion system, and we deployed them with participants to measure their ability to change the belief of participants in a persuasion goal. We excluded the participants from data gathering studies from taking part in the persuasion experiments. The two versions followed the protocol described in Section \ref{section:protocol}, and were implemented as follows.


\begin{description}
\item[Baseline system] This  was the baseline or control system, and it chose arguments at random from the ones attacking at least one of the arguments presented by the persuadee in the previous step.
\item[Advanced system]
This was the system that made a strategic choice of move that maximizes the reward (see Section \ref{sec:advancedstrategy}. It incorporates the Monte Carlo Tree Search algorithm as presented in Section \ref{section:montecarlo} and uses the reward function as presented in Section \ref{section:reward}. 
\end{description}


In this study, we used 261 participants recruited from the Prolific crowdsourcing website, which later allowed us to have 126 
participants for the advanced system and 119 for the baseline\footnote{Some of the submissions had to be filtered out due to technical issues.}.
The prescreening required their first language to be English, nationality British, and age between 18 and 70. 


At the start of each experiment, each participant was asked the same TIPI, demographic and domain dependant questions as in the ranking of concerns explained in Section \ref{sub:tree}.
The full survey description and demographic statistics can be found in the Data Appendix.

After collecting the demographic and personality information, we asked the participants for their opinion on the following statement (using a slider bar ranging from -3 to 3 with 0.01 graduation). 
We note that the answer $0$ (\emph{i.e.} 
neither agree nor disagree) was not permitted - we requested the participants to express their preference, 
independently of how small it may be.

\begin{quote}
Are you against (slider to the left) or for (slider to the right) the abolishment of the \pounds 9K fees for universities, and to what degree?
\end{quote}

Then we presented each participant with a chatbot (either the baseline system or the advanced system) composed of a front-end we coded in Javascript and a back-end in C++ served by an API in Python using the Flask web server library\footnote{The code is available at \url{https://github.com/ComputationalPersuasion/MCCP}.} (see high level architecture in Figure \ref{fig:architecture}). 
The Javascript front-end gathered the arguments selected by the participant and sent them to the Python API. 
They were transparently forwarded to the C++ back-end to calculate the best answer to these arguments.
Then the back-end sent the system the answer and the allowed counterarguments back to the API that translated it to text and sent it to the front-end to be presented to the participant.
After the end of the dialogue with the chatbot, the participant was again presented the statement about the abolition 
of the \pounds 9K student fee, and asked to express their belief using the slider bar. 
This way we obtain a value for the participant's belief before and after the persuasion dialogue. 

\begin{figure}[t]
\begin{center}
\begin{tikzpicture}[->,>=latex]
\node[text centered,shape=rectangle,minimum size=1cm,text width = 6cm,draw] (c)  at (0,4) {User interface (Javascript)};
\node (c1) at (-1,3.6) {};
\node (c2) at (1,3.6) {};
\node (b1) at (-1,2.4) {};
\node (b2) at (1,2.4) {};
\node[text centered,shape=rectangle,minimum size=1cm,text width = 6cm,draw] (b)  at (0,2) {API (Python+Flask webserver)};
\node (b3) at (-1,1.6) {};
\node (b4) at (1,1.6) {};
\node (a1) at (-1,0.4) {};
\node (a2) at (1,0.4) {};
\node[text centered,shape=rectangle,minimum size=1cm,text width = 6cm,draw] (a)  at (0,0) {Dialogue engine (C++)};
\path (c1) edge (b1);
\node (n1) at (-2.5,3) {User arguments};
\path (b2) edge (c2);
\node (n2) at (2.5,3) {System arguments};
\path (b3) edge (a1);
\node (n3) at (-2.5,1) {User arguments};
\path (a2) edge (b4);
\node (n4) at (2.5,1) {System arguments};
\end{tikzpicture}
\end{center}
\caption{High level architecture of chatbot platform.}
\label{fig:architecture}
\end{figure}
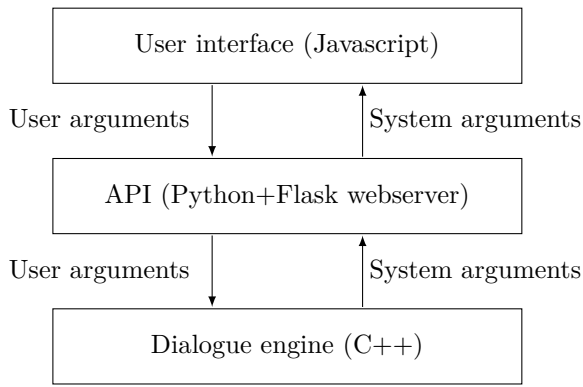

For the MCTS component of the advanced APS, we set the number of simulations to 1000 to balance out the trade-off between the completeness of exploration and 
the time waited by the participant. 
Indeed, the longer they wait, the less engaged they are, which causes deterioration in the quality of data.
On average, the round trip from sending the counterarguments picked by the user to the back-end through the API, calculating the answer and back to the front-end for presentation (therefore including network time and client side execution) took between 0.5 and 5 seconds, depending on the number of counterarguments selected by the participant.
We argue that these are acceptable times compared to traditional human to human chat experience.

\subsection{Results}


From the dialogues we obtained from running the advanced system and the baseline system, we obtained a head-to-head comparison (Section \ref{section:results:comparison}).
All the dialogues are presented in the Data Appendix. 

\subsubsection{Structural analysis of dialogues} 

We start by considering the structure of the dialogues produced by our APSs. We focus on two dimensions: completeness and linearity. 

By a complete dialogue we understand a dialogue such that all the leaves in the subgraph associated with the dialogue are of even depth in the original graph on which the dialogue was based, and no $\nrej$ option was selected. In other words, it is a dialogue in which all arguments of the participants are countered by the system. 

By a linear dialogue we understand a dialogue such that the subgraph associated with a dialogue is simply a chain from the root to the leaf. In other words, at most one argument is used in every dialogue move.

In addition to this, we split our analysis with respect to the graph that was used for the dialogue. We have the graph built in favour of keeping the university fees (used when the participant was in favour of their abolishment), and the dual abolishing graph (used when the participant was in favour of keeping the fees). 

Tables \ref{tab:structureanalysisADV} and \ref{tab:structureanalysisBAS} show the distributions of dialogues structure of different types that were produced by the advanced and baseline systems.


\begin{table}[ht]
\centering
\begin{tabular}{rcccrr}
 \addlinespace[-\aboverulesep] 
 \cmidrule[\heavyrulewidth]{2-6}
& Complete & Linear & Keeping Graph & \# of dialogues & \% of dialogues \\
\cmidrule{2-6}
& $\checkmark$ & $\checkmark$ & $\checkmark$ & 62 & 49.21\% \\
& $\checkmark$ &              & $\checkmark$ & 23 & 18.25\% \\
& $\checkmark$ & $\checkmark$ &              & 14 & 11.11\% \\
&              & $\checkmark$ & $\checkmark$ & 10 & 7.94\% \\
& $\checkmark$ &              &              & 5 & 3.97\% \\
&              &              & $\checkmark$ & 5 & 3.97\% \\
&              &              &              & 5 & 3.97\% \\
&              & $\checkmark$ &              & 2 & 1.59\% \\
\midrule
\multicolumn{1}{r}{\# of dialogues} & 104 & 88 & 100 & 126 & \multicolumn{1}{r}{}  \\
\multicolumn{1}{r}{\% of dialogues} & 82.54\% & 69.84\%& 79.37\% &\multicolumn{1}{r}{} & \multicolumn{1}{r}{} \\
 \cmidrule[\heavyrulewidth]{1-5}
\end{tabular} 
\caption{Structural analysis of the dialogues with the advanced system.}
\label{tab:structureanalysisADV}
\end{table}


\begin{table}[ht]
\centering
\begin{tabular}{rcccrr}
\addlinespace[-\aboverulesep] 
 \cmidrule[\heavyrulewidth]{2-6}
& Complete & Linear & Keeping Graph & \# of dialogues & \% of dialogues \\
\cmidrule{2-6}
& $\checkmark$ & $\checkmark$ & $\checkmark$ & 54 & 45.38\% \\
& $\checkmark$ &              & $\checkmark$ & 7 & 5.88\% \\
& $\checkmark$ & $\checkmark$ &              & 14 & 11.76\% \\
&              & $\checkmark$ & $\checkmark$ & 12 & 10.08\% \\
& $\checkmark$ &              &              & 2 & 1.68\% \\
&              &              & $\checkmark$ & 21 & 17.65\% \\
&              &              &              & 7 & 5.88\% \\
&              & $\checkmark$ &              & 2 & 1.68\% \\
\midrule
\multicolumn{1}{r}{\# of dialogues} & 77 & 82 & 94 & 119 & \multicolumn{1}{r}{}  \\
\multicolumn{1}{r}{\% of dialogues} & 64.71\% & 68.91 \%& 78.99\% &\multicolumn{1}{r}{} & \multicolumn{1}{r}{} \\
\cmidrule[\heavyrulewidth]{1-5}
\end{tabular} 
\caption{Structural analysis of the dialogues with the baseline system.}
\label{tab:structureanalysisBAS}
\end{table}
 
In the case of both APSs we therefore observe that the dialogues are more likely to be complete than incomplete, and linear rather than non-linear (i.e. branched). We also note that in both cases, the majority of the participants were in favour of abolishing the university fees, and therefore the graph in favour of keeping them was used for the dialogue - this is however something out of the control of the APS and merely reflects the views of the participant pool. 

There are however some differences between the dialogues produced by the two systems. We first observe that while complete discussions are prevalent in both the advanced and the baseline system, there is still quite a difference as to their degree. Deeper analysis of the reasons for this variance in incompleteness (see Table \ref{tab:compltenessanalysis}) shows that the primary cause of this is associated with the baseline system being less likely to address all of the user's arguments (i.e. the system is more likely to lead to odd-branched dialogues).

We believe this behaviour is simply a result of the design of the baseline system, where counterarguments are selected from the applicable ones at random. This behaviour might also contribute to the differences in the distributions of different kinds of dialogues produced by the APSs - while complete and linear dialogues are prevalent in both cases, there are differences further down. In particular, incomplete and nonlinear dialogues rank second for the baseline system, while complete and nonlinear dialogues rank second for the advanced system.

\begin{table}[ht]
\centering
\begin{tabular}{llrrr}
\addlinespace[-\aboverulesep] 
 \cmidrule[\heavyrulewidth]{3-5}
\multicolumn{2}{r}{} & Contains $\nrej$ & Contains odd branch & Contains both \\
 \midrule
\multirow{2}{*}{Advanced} 
 & as \% of dialogues & 15.08\% &	2.38\%&	0.00\% \\
& \vtop{\hbox{\strut as \% of incomplete}\hbox{\strut  dialogues}}
 & 86.36\%	& 13.64\% &	0.00\% \\
 \midrule
\multirow{2}{*}{Baseline}  
 & as \% of dialogues  & 18.49\% &	19.33\%	 &2.52\%
 \\
& \vtop{\hbox{\strut as \% of incomplete}\hbox{\strut  dialogues}}
&52.38\% &	54.76\%	& 7.14\% \\
 \bottomrule
\end{tabular} 
\caption{Reasons for incompleteness of dialogues in the advanced and baseline systems.}
\label{tab:compltenessanalysis}
\end{table}

An additional message to take out of this is that $\nrej$ moves had a big impact on incompleteness. This essentially means that the constructed argument graphs should have contained more arguments - while the current ones were sufficient for the majority of the participants, there is still room for improvement. 


\subsubsection{Comparing the system and the baseline}
\label{section:results:comparison} 

A natural next step is to compare the performance of the advanced and baseline systems, by which we understand the belief changes they lead to.

The scatter plots for the before-after beliefs in persuasion goal for each of our APSs are visible in Figure \ref{fig:beforeafterplots}. Please note that a \enquote{perfect} system, managing to convince the participant to radically change their stance on the persuasion goal, would be one producing an \enquote{after} belief of 3 for any participant with a negative \enquote{before} belief, and an \enquote{after} belief of -3 for any participant with a positive \enquote{before} belief (i.e. the perfect scatter plot would not be diagonal). 

\begin{figure}[!ht] 
\centering
\pgfplotsset{every axis/.style={xmin=-3, xmax = 3, ymin = -3, ymax = 3, width=0.5\textwidth,
grid = major,
x label style ={at={(axis description cs:0.5,0.04)}},
y label style ={at={(axis description cs:0.08,0.5)}}, 
xtick ={-3,-2,-1,0,1,2,3},
ytick ={-3,-2,-1,0,1,2,3},
xticklabels ={-3,-2,-1,0,1,2,3},
yticklabels={-3,-2,-1,0,1,2,3}, 
enlarge x limits=0.05,
enlarge y limits=0.05,
    xlabel=$Before$,
    ylabel=$After$,}}
\begin{tikzpicture}
\begin{axis}[name=adv, title = {Advanced System}]
    \addplot[only marks, mark=o]
    coordinates{ 
(0.7,0.13)
(1.29,1.41)
(2.59,2.45)
(3,2.58)
(1.29,1.15)
(2.4,2.46)
(-2.77,-3)
(2.94,2.9)
(3,2.74)
(2.03,1.8)
(2.79,1.59)
(1.59,2)
(0.5,0.17)
(-1.12,-1.11)
(1.76,2.58)
(1.54,1.82)
(3,3)
(0.59,0.72)
(1.81,1.49)
(2.97,3)
(3,1.67)
(3,2.97)
(-0.71,-0.72)
(2.99,3)
(3,2.96)
(0.09,0.17)
(-1.34,-0.77)
(2.99,2.93)
(2.95,2.96)
(2.97,2.96)
(3,2.95)
(0.69,0.92)
(2.62,2.71)
(1.94,1.58)
(2.57,2.4)
(3,3)
(2.05,3)
(2.94,2.61)
(2.94,2.99)
(-1.41,-1.79)
(3,3)
(-1.78,-1.69)
(2.92,2.91)
(2.48,2.51)
(0.59,0.19)
(2.44,2.54)
(3,1.57)
(2.99,3)
(-2.28,-1.63)
(1.21,0.31)
(1.99,2.16)
(2.99,3)
(1.37,1.17)
(-1.81,-1.38)
(-0.96,-1.08)
(3,3)
(-0.55,-0.19)
(2.7,2.86)
(3,2.95)
(-0.75,-0.52)
(2.98,2.85)
(1.66,1.2)
(3,2.91)
(1.35,1.56)
(1.23,1.48)
(1.94,1.66)
(3,3)
(3,2.79)
(-1.46,-1.87)
(2.97,1.95)
(2.68,3)
(-0.48,-0.55)
(1.11,0.74)
(2.36,1.88)
(-3,-3)
(-0.85,-0.86)
(1.05,0.68)
(1.73,1.38)
(2.73,2.94)
(-0.12,-0.25)
(1.34,1.76)
(1.46,1.58)
(0.46,0.16)
(3,2.95)
(0.79,0.58)
(2.92,2.94)
(2.64,2.67)
(1.59,1.12)
(0.95,0.69)
(3,2.98)
(2.06,2.4)
(1.05,1.36)
(-2.29,2.83)
(1.56,1.77)
(2.77,2.93)
(3,1.86)
(2.98,2.99)
(2.98,3)
(1.76,1.74)
(2.42,2.45)
(-1.21,-2.42)
(3,3)
(2.2,2.23)
(3,2.97)
(1.19,1.16)
(3,3)
(1.64,1.12)
(0.07,0.05)
(-0.43,-0.11)
(2.36,2.73)
(3,3)
(1.27,0.57)
(0.78,0.71)
(2.93,3)
(-1.22,-1.29)
(0.87,0.54)
(-0.73,0.26)
(1.41,0.05)
(2.28,2.39)
(1.78,1.57)
(-2.97,-2.98)
(2.39,2.78)
(-2.81,-3)
(3,1.79)
(-1.21,-1.25)
(1.83,1.79)
(1.83,0.71)
(1.5,1.61)
(-3,-3)
(3,3)
(-1.75,-1.82) 
    }; \label{plot_advanced} 
  \end{axis}

\begin{axis}[name=bas, at=(adv.right of south east), xshift = 1.2cm,title = {Baseline System}]
    \addplot[only marks, mark=o]
    coordinates{ 
(1.87,1.98)
(2.07,2)
(-1.24,-1.66)
(-2.65,-2)
(2.98,3)
(2.07,2.35)
(2.6,2.53)
(2.14,2.3)
(-2.19,-1.74)
(0.91,1.17)
(1.18,1.14)
(2.96,2.73)
(2.95,2.81)
(3,3)
(3,2.99)
(-1.91,-1.74)
(1.99,1.3)
(3,2.74)
(1.42,1.28)
(-1.37,-1.93)
(2.36,2.39)
(2.89,2.82)
(1.41,1.6)
(0.89,0.69)
(2.52,2.54)
(3,3)
(1.16,1.74)
(3,3)
(2.79,2.94)
(1.75,1.35)
(-1.89,-1.72)
(1.85,1.46)
(3,2.94)
(-2.1,-2.76)
(1.58,1.31)
(2.99,2.99)
(2.6,1.38)
(2.96,3)
(3,2.98)
(2.13,1.79)
(-1.01,-1.51)
(1.87,1.91)
(2.59,2.8)
(1.05,1.3)
(-0.83,-1.34)
(1.47,1.29)
(1.95,1.91)
(2.22,2.29)
(3,3)
(-3,-3)
(2.76,2.38)
(3,3)
(1.03,1.11)
(-1.87,-1.75)
(3,2.96)
(-1.63,-1.46)
(-3,-3)
(3,3)
(3,3)
(2.38,2.02)
(1.41,1.51)
(1.2,1.13)
(1.19,1.32)
(0.15,0.07)
(1.03,-0.43)
(1.25,2.2)
(1.33,0.5)
(1.7,1.71)
(-2.69,-2.33)
(3,1.65)
(2.93,2.96)
(1.43,2.23)
(-2.59,-2.65)
(2.49,2.53)
(3,3)
(3,3)
(3,3)
(3,2.5)
(-2.18,-2.18)
(3,3)
(-0.78,-0.43)
(2.03,1.94)
(1.36,2.03)
(2.94,2.98)
(1.99,1.49)
(1.57,1.94)
(3,3)
(1.94,1.26)
(-1.29,0.11)
(-3,1.24)
(0.85,1.48)
(-2.18,-2.04)
(3,2.99)
(3,3)
(2.47,2.71)
(1.13,1.32)
(1.1,1.3)
(2.75,2.91)
(3,2.92)
(2.96,3)
(2.37,2.1)
(0.41,0.8)
(0.91,1.1)
(3,3)
(2.21,2.07)
(2.61,1.29)
(1.91,1.94)
(1.71,1.52)
(1.04,1.32)
(2.3,3)
(3,3)
(2.64,2.37)
(-1.23,-0.29)
(-2.72,-2.25)
(2.51,1.38)
(2.98,2.93)
(2.2,1.53)
(-1.03,-0.52)
(2.81,3)
(2.05,1.84)
(0.74,0.87)
(1.15,0.15)
(-2.49,-2.27)
(0.9,0.98)
(2.34,1.08)
(1.2,1.24)
(3,2.96)
(2.27,2.21)
(-0.42,-0.21)
(-1.71,-1.9)
    }; \label{plot_baseline} 
  \end{axis}
\end{tikzpicture}
\caption{Scatter plots of the before and after beliefs for participants that entered a dialogue with the 
advanced or with the baseline system.}
\label{fig:beforeafterplots}
\end{figure}
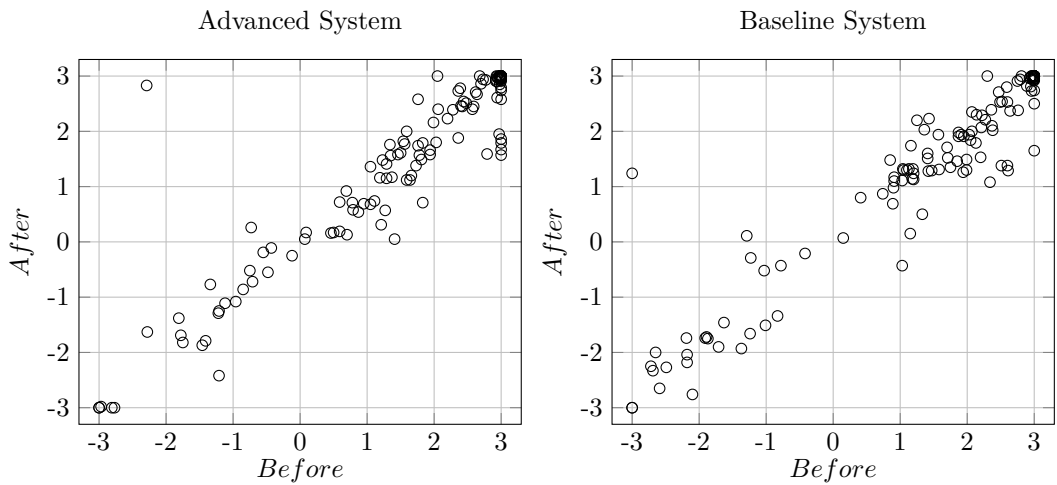

We supplement the scatterplots with additional statistical analysis of the belief change.
We can observe that the distributions of beliefs before the dialogues in the advanced and baseline systems are not statistically different, meaning that the two populations are not dissimilar belief-wise for the two systems\footnote{We carried out a Mann-Whitney U test and obtained p-value of 0.71138. We used \url{https://www.socscistatistics.com/tests/mannwhitney/default2.aspx} for the calculations.}

However, the effects that these systems had on the users are not the same, as visible in Table \ref{tab:statnorm}. We have used Shapiro-Wilk test in order to determine whether the \enquote{before} beliefs of the participants were normally distributed. If the answer was positive, student t-test was used to determine if the changes in beliefs were significant or not; if the answer was negative, Wilcoxon signed-rank test was used\footnote{All calculations have been carried out in R, detailed statistics can be found in the appendix at the end of this paper.}. 

We observe that for the baseline system, independently of the considered subclass of dialogues, the changes in beliefs were either not significant or significance could not have been determined. In turn, the advanced system led to statistically significant changes in beliefs in most cases where significance could have been determined.



\begin{table}[ht] 
\begin{tabularx}{\textwidth}{XXXXX} 
\addlinespace[-\aboverulesep] 
 \cmidrule[\heavyrulewidth]{2-5}
\multicolumn{1}{X}{}	& \multicolumn{2}{c}{Advanced System}& \multicolumn{2}{c}{Baseline System} \\
	Dialogue Type	&	Normality	&	Significance &	Normality	&	Significance	\\ 
\midrule
All	                &	x	        &	\checkmark	&	x			&	x	\\
Complete	        &	x	        &	\checkmark	&	x			&	x	\\
Incomplete  	    &	x	        &	-	        &	x			&	-	\\
Linear	            &	x	        &	\checkmark	&	x			&	x	\\
Nonlinear       	&	x	        &	-	        &	x			&	-	\\
Keeping Graph   	&	x	        &	\checkmark	&	x			&	x	\\
Abolishing Graph	&	\checkmark	&	x	        &	\checkmark	&	x	\\
\bottomrule
\end{tabularx}  
\caption{Results of analysis of statistical significance of belief changes caused by the APSs on different types of dialogues. Shapiro-Wilk test was used to determine whether the \enquote{before} beliefs were normally distributed. If they were, t-test was used to determine significance of belief changes; otherwise, Wilcoxon signed-rank test was used. By \enquote{-} we understand that due to the nature of the data, exact p-value could not have been computed, and we make no claims about the significance. Further details can be found on Table \ref{tab:statnormdet}.}
 \label{tab:statnorm}
\end{table}

Despite these positive results, the actual differences in belief changes between the two systems are rather modest when we look at the numbers. Table \ref{tab:avgChange} shows the average changes in beliefs between the two system. We distinguish between the normal average, which simply takes the mean of all differences in beliefs (which can be negative or positive depending on the persuasion goal), and the absolute average, which ignores whether the change is negative or positive and focuses on the value only. Additionally, in Tables 
\ref{tab:changeADV} and \ref{tab:changeBAS} we include the percentage distribution of the changes.
 

\begin{table}[ht]
\begin{tabularx}{\textwidth}{XXXXX} 
\addlinespace[-\aboverulesep] 
 \cmidrule[\heavyrulewidth]{2-5}
\multicolumn{1}{X}{}	& \multicolumn{2}{c}{Advanced System}& \multicolumn{2}{c}{Baseline System} \\
Dialogue Type	&	Average change	&	Average absolute change	&	Average change	&	Average absolute change	\\ 
\midrule
All					&	0.162	&	0.308	&	0.139	&	0.321	\\ 
Complete			&	0.165	&	0.313	&	0.203	&	0.361	\\ 
Incomplete			&	0.152	&	0.285	&	0.022	&	0.248	\\ 
Linear				&	0.216	&	0.348	&	0.148	&	0.342	\\ 
Nonlinear			&	0.038	&	0.217	&	0.119	&	0.273	\\ 
Keeping Graph		&	0.146	&	0.271	&	0.094	&	0.264	\\ 
Abolishing Graph	&	0.224	&	0.451	&	0.309	&	0.536	\\
\bottomrule
\end{tabularx}  
\caption{Average changes in beliefs in dialogues with the advanced and baseline systems. A change in beliefs can be negative or positive depending on the persuasion goal. Average changes takes the average of all the obtained differences; absolute average ignores whether the change is negative or positive and focuses on the value. Results are rounded to the third decimal place.}
 \label{tab:avgChange}
\end{table}
 

\begin{table}[ht!]
\begin{tabularx}{\textwidth}{XXXXXXXX} 
\addlinespace[-\aboverulesep] 
 \cmidrule[\heavyrulewidth]{3-7}
\multicolumn{2}{X}{}		& $++$	&	$+$	&	x	&	$-$	&	$- -$	\\
\multicolumn{1}{X}{}	&Population \%		&	[1,3]	&	(0,1)	&	0	&	(-1,0)	&	[-3,-1]	\\
\midrule
All	&	100\%	&	7.14\%	&	44.44\%	&	7.14\%	&	40.48\%	&	0.79\%	\\
Complete	&	82.54\%	&	5.77\%	&	50\%	&	6.73\%	&	36.54\%	&	0.96\%	\\
Incomplete	&	17.46\%	&	13.64\%	&	18.18\%	&	9.09\%	&	59.09\%	&	0\%	\\
Linear	&	69.84\%	&	9.09\%	&	44.32\%	&	6.82\%	&	39.77\%	&	0\%	\\
Nonlinear	&	30.16\%	&	2.63\%	&	44.74\%	&	7.89\%	&	42.11\%	&	2.63\%	\\
Keeping Graph	&	79.37\%	&	8\%	&	47\%	&	7\%	&	38\%	&	0\%	\\
Abolishing Graph	&	20.63\%	&	3.85\%	&	34.62\%	&	7.69\%	&	50\%	&	3.85\%	\\
\bottomrule
\end{tabularx}
\caption{Belief change analysis of the dialogues with the advanced system. The results are rounded to two decimal places. Second row shows the interval to which a belief change value should belong to be interpreted as very positive ($++$), positive ($+$), negative ($-$), very negative  ($--$), or no change ($x$). }
 \label{tab:changeADV}
\end{table}


 \begin{table}[ht!] 
\begin{tabularx}{\textwidth}{XXXXXXXX} 
\addlinespace[-\aboverulesep] 
 \cmidrule[\heavyrulewidth]{3-7}
\multicolumn{2}{X}{}		& $++$	&	$+$	&	x	&	$-$	&	$- -$	\\
\multicolumn{1}{X}{}	&Population \%		&	[1,3]	&	(0,1)	&	0	&	(-1,0)	&	[-3,-1]	\\
\midrule
All	&	100\%	&	7.56\%	&	41.18\%	&	14.29\%	&	36.97\%	&	0\%	\\
Complete	&	64.71\%	&	10.39\%	&	38.96\%	&	14.29\%	&	36.36\%	&	0\%	\\
Incomplete	&	35.29\%	&	2.38\%	&	45.24\%	&	14.29\%	&	38.10\%	&	0\%	\\
Linear	&	68.91\%	&	7.32\%	&	41.46\%	&	10.98\%	&	40.24\%	&	0\%	\\
Nonlinear	&	31.09\%	&	8.11\%	&	40.54\%	&	21.62\%	&	29.73\%	&	0\%	\\
Keeping Graph	&	78.99\%	&	7.45\%	&	37.23\%	&	14.89\%	&	40.43\%	&	0\%	\\
Abolishing Graph	&	21.01\%	&	8\%	&	56\%	&	12\%	&	24\%	&	0\%	\\
\bottomrule
\end{tabularx}
\caption{Belief change analysis of the dialogues with the baseline system. The results are rounded to two decimal places. Second row shows the interval to which a belief change value should belong to be interpreted as very positive ($++$), positive ($+$), negative ($-$), very negative  ($--$), or no change ($x$).  }
 \label{tab:changeBAS}
\end{table}


An important thing to notice here is the behaviour of the participants in the complete dialogues. The arguments uttered by the APS in these scenarios correspond to the grounded/preferred/stable extensions of the graphs associated with the dialogues. Predictions using classical Dung semantics would put both of our APS in a strongly winning position, while we would argue that it is not the case here. 

Despite these properties of the complete dialogues, we achieve no statistical significance of the changes in beliefs for the baseline system (see Table \ref{tab:statnorm}). This important threshold is only passed by the advanced system, which, in contrast to the baseline APS, attempts to tailor the dialogue to the profile of the user. This indicates that relying only on the structure of the graph for selecting arguments in dialogues is insufficient. Presenting just any counterargument to a participant's argument turned out to be ineffective in the context of this experiment. In contrast, including beliefs and concerns as factors in the selection of the counterarguments, proved to have statistically significant effects. 

Nevertheless, we still need to acknowledge that in pure numbers, the results are modest. The average changes in beliefs require improvement. While 55.77\% of the users that engaged with the advanced system did experience positive changes, this still leaves a significant proportion of participants that were affected negatively or not affected at all. Thus, there is the need for increasing both the number of participants experiencing positive changes, as well as the degree of these changes. 

We therefore believe that further research in this direction needs to be undertaken. The results require improvement and/or replication; there may also exist other additional factors, besides their beliefs and concerns, that would allow for dialogues to be better tailored for participants.

\subsection{Conclusions}

In general, the changes in beliefs (good direction, bad direction, etc) are quite similar between the baseline and advanced systems. The averages are also relatively close. Yet, only in the advanced system, do we get a statistically significant change in users' beliefs.  Even if we focus on the simpler dialogues (i.e. those that are complete, which are all dialogues that a classical argumentation system would create), the results are similar. Yet again, only the advanced system is significant. So the results suggest that the advanced system is better at changing belief more in favour of the persuasion goal than the baseline system. 

Our claim at the start of this study was that a dialogue needs to be tailored to the user, otherwise it is less effective than it could be. Lack of significance in the baseline system supports that, as no tailoring was taking place. The advanced system was doing some tailoring, and we obtain significance. Nevertheless, the end gain is not as marked as we might hope for. We get a little less than 5\% population increase in positive changes, and over 1\% increase in negative changes. 
Despite this, we need to remember that it is widely acknowledged that convincing anyone of anything (when they are allowed to have their own opinions, as opposed to a logical reasoning exercise), is a difficult task. So developing a system that is going to persuade the majority of participants on a real and controversial topic such as university fees is not very likely. Therefore, even these small improvements of the advanced system over the baseline system are valuable, and indicate that further research into dialogue tailoring approaches is promising.




\section{Literature review}
\label{section:literaturereview}


Since the original proposals for formalizing dialogical argumentation \cite{Ham71,Mac79}, a number of proposals have been made for various kinds of protocol (\emph{e.g.}, \cite{ME98,AMP00,AMP00b,dignum:00,HMP01,MP02,MP02b,MEPA03,Prakken2005,Prakken06,CaminadaPodlaszewki12}).  More recently, interest has been focused on strategies. In the following, we review some of the key work on strategies. We also refer to \cite{Thimm14} for a review of strategies in multi-agent argumentation.


Some strategies have focussed on correctness with respect to argumentation undertaken directly with the knowledge base, in other words, whether the argument graph constructed from the knowledge base yields the same acceptable arguments as those from the dialogue (\emph{e.g.}, \cite{BH09,FanToni11}).
Strategies in argumentation have been analysed using game theory (\emph{e.g.}, \cite{RahwanLarson08,RL09}, \cite{FanToni12}), but these are more concerned with issues of mechanism design, rather than persuasion. 
 In these papers, the game is a one step process rather than a dialogue, and they are concerned with manipulation rather than persuasion.


In \cite{BCB14}, a planning system is used by the persuader to optimize choice of arguments based on belief in premises, and in \cite{BCH2017}, an automated planning approach is used for persuasion that accounts for the uncertainty of the proponent's model of the opponent by finding strategies that have a certain probability of guaranteed success no matter which arguments the opponent chooses to assert.
Alternatively, heuristic techniques can be used to search the space of possible dialogues \cite{Murphy2016}. 
Persuasion strategies can also be based on convincing participants according to what arguments they accept given their view of the structure of an argument graph \cite{Murphy2018}. 
As well as trying to maximize the chances that a dialogue is won according to some dialectical criterion, a strategy can aim to minimize the number of moves made \cite{Atkinson2012}. 
The application of machine learning is another promising approach to developing more sophisticated strategies such 
as the use of reinforcement learning \cite{Huang2007,Monteserin2012,Rosenfeld2016ecai,Alahmari2017,Rach2018,Katsumi2018}
and transfer learning \cite{Rosenfeld2016}.

There are some proposals for strategies using probability theory to, for instance, select a move based on what an agent believes the other is aware of \cite{RienstraThimmOren13}, or, to approximately predict the argument an opponent might put forward based on data about the moves made by the opponent in previous dialogues \cite{HadjinikolisSiantosModgilBlackMcBurney13}.
Using the constellations approach to probabilistic argumentation, a decision-theoretic lottery can be constructed for each possible move \cite{HunterThimm2016ijar}.
Other works represent the problem as a probabilistic finite state machine with a restricted protocol \cite{Hunter14Bis}, and generalize it to POMDPs when there is uncertainty on the internal state of the opponent \cite{HadouxBeynierMaudetWengHunter15}. 
POMDPs are in a sense more powerful than the MCTS that we advocate in our proposal. However, as discussed in \cite{HadouxBeynierMaudetWengHunter15}, there is a challenge in managing the state explosion in POMDPs that arises from modelling opponents in argumentation.


A novel feature of the protocol in this study is that it allows a form of incompleteness in that not every user argument has to be countered for the dialogue to continue. 
The protocol ensures that for each user argument presented at each step of the dialogue, if the system has a counterargument to it that has not been presented in the dialogue so far, then it will present a counterargument to that user argument. However, if the system does not have a counterargument for a user argument, but it does have a counterargument for another user argument played at that step of the dialogue, then the dialogue can still continue. The aim of this tolerance to incompleteness is to reflect how real-life discussions do not necessitate every argument to be countered. This means that discussions can usefully continue, and they are not sensitive to a participant be able to counter every possible argument from the other participants. So our protocol is in contrast to protocols in other approaches to computational argumentation where the aim is for each participant to counter all the arguments by the other agent.

In our previous work \cite{HadouxHunter2018submission}, we developed an APS that selects arguments based on the concerns of the participant in the current dialogue. For this, we assumed that we have a set of arguments, and that each argument is labelled with the type(s) of concern it has an impact on. Furthermore, the system had a model of the user in the form of a preference relation over the types of concern. We did not assume any structure for the preference relation. In particular, we did not assume it is transitive. For each user argument $A$ that the system wishes to attack with a counterargument, the set of attackers (the set of candidates) is identified (\emph{i.e.}, the set of arguments $B$ such that $(B,A) \in \arcs(\graph)$). From this set of candidates, the most preferred one is selected. In other words, the argument returned was the most preferred  attacker of $A$ according to the preference over concerns. In a study with 100 participants,  the results showed that preferences over concerns can indeed be used to improve the persuasiveness of a dialogue when compared with randomly generated dialogue. In another study \cite{Chalaguine2020}, using over 1000 crowdsourced arguments on university fees, taking concerns were shown to be effective in a chatbot that used a natural language interface (i.e. the user could type their arguments in natural language, and the chatbot located the best matching argument in the argument graph, and then responded with a counterargument that appeared to meet the concerns raised in the user's argument).

Conceptually, concerns can be seen as related to value-based argumentation. This approach takes the values of the audience 
into account, where a value is typically seen as moral or ethical principle that is promoted by an argument. It can  
also be used to capture the general goals of an agent, as discussed in \cite{Atkinson2006}. 
A value-based argumentation framework (VAF) extends an abstract argumentation framework by assigning a value to each argument, and for each type of audience, a preference relation over values. This preference relation which can then be used to give a preference ordering over arguments \cite{BenchCapon2003,BenchCapon2005,Atkinson2006,BenchCapon2009,AtkinsonWyner2013}. 
The preference ordering is used to ignore an attack relationship when the attackee is more preferred than the attacker, for that member of the audience. This means the extensions obtained can vary according to who the audience is. 
VAFs have been used in a dialogical setting to make strategic choices of move \cite{BenchCapon2002}.
So theoretically, VAFs could take concerns into account, but they would be unable to model beliefs.There is also no decision-theoretic framework for this, nor is there an empirical evaluation of VAFs. 




More recently, the use of values has been proposed for labelling arguments that have been obtained by crowdsourcing. Here a value is a category of motivation that is important in the life of the agent (\emph{e.g.}, family, comfort, wealth, etc.), and a value assignment to an argument is a category of motivation for an agent if she were to posit this argument \cite{CHHP2018}. It was shown with participants that different people tend to apply the same (or similar) values to the same argument. 
This provides additional evidence that the use of concerns is meaningful and practical, thus 
further supporting our methodology.


The notion of interests as arising in negotiation is also related to concerns. In psychological studies of negotiation, it has been shown that it is advantageous for a participant to determine which goals of the other participants are fixed and which are flexible \cite{Fisher1981}. In \cite{Rahwan2009}, this idea was developed into an argument-based approach to negotiation where meta-information about each agent's
underlying goals can help improve the negotiation process. 
Argumentation has been used in another approach to co-operative problem solving where intentions are exchanged between agents as part of dialogue involving both persuasion and negotiation \cite{dignum:00}. 
Even though the notions of interests and intentions are used in a different way to the way we use the notion of concerns in this paper, it would be worthwhile investigating the relationship between these concepts in future work.

The empirical approach taken in this paper is part of a trend in the field of computational argumentation for studies with participants. This includes studies that evaluate the accuracy of dialectical semantics of abstract argumentation for predicting behaviour of participants in evaluating arguments \cite{Rahwan2011,Cerutti2014}, 
studies comparing a confrontational approach to argumentation with argumentation based on appeal to friends, appeal to group, or appeal to fun \cite{Vargheese2013,Vargheese2016},
studies of appropriateness of probabilistic argumentation for modelling aspects of human argumentation \cite{PolbergHunter2018ijar}, 
studies to investigate physiological responses of argumentation \cite{Villata2017}, 
studies using reinforcement learning for persuasion \cite{Huang2007}, 
and studies of the use of predictive models of an opponent in argumentation to make strategic choices of move by the proponent \cite{Rosenfeld2016}. 
There have also been studies in psycholinguistics to investigate the effect of argumentation style on persuasiveness \cite{Lukin2017}. 

There have already been some promising studies that indicate the potential of using automated dialogues in behaviour change such as using dialogue games
for health promotion \cite{Grasso1998,Cawsey1999,Grasso2000,Grasso2003},
conversational agents for encouraging exercise \cite{Nguyen2007,Bickmore2013} 
and for promoting plant-based diets \cite{Zaal2017}, 
dialogue management for persuasion \cite{Andrews2008},
persuasion techniques for healthy eating messages \cite{JosekuttyThomas2017},
and tailored assistive living systems for encouraging exercise \cite{Guerrero2016}.
However, none of these studies have provided a framework for strategic argumentation, in contrast to the proposal we present in this paper.



\section{Discussion}
\label{section:discussion}

In this paper, we have presented a framework for user modelling that incorporates the beliefs and concerns of persuadees, as well as introduced a framework for optimizing the choice of moves in dialogical argumentation by taking into account these user models. We have shown how we can crowdsource the data required for constructing the user models, and that this can be used by APSs to make strategic choices of move that outperform a baseline system over a population of participants.

This study therefore indicates the value of taking the beliefs and concerns of agents into account. 
Furthermore, it indicates the viability and utility of undertaking real-time decisions on moves to make based on the 
Monte Carlo Tree Search algorithm. The way we have harnessed this algorithm is quite general, and alternative options for the reward 
function could be deployed. For instance, the belief component of the reward function could be replaced by 
different methods of modelling belief updates (e.g. \cite{Hunter2016sum,HunterPotyka2017}), or even having a richer modelling of user beliefs and updating of them (e.g. \cite{HunterPotykaPolberg2019,HunterPolbergThimm2018}).
One can also consider different concern scoring 
approaches, where rather than using preferences, we can focus on how the concerns associated with the played arguments 
are matched between the system and user moves or investigate different quantity-quality balancing. 
Finally, the way these two score values are aggregated can also be modified. In future work, we will 
investigate some of the options to get a more comprehensive understanding of the effectiveness and behaviour of 
different reward functions. 



Another topic for future work is the specification of the protocol.  Many protocols for dialogical argumentation involve a depth-first approach (e.g., \cite{BenchCapon2002}). So when one agent presents an argument, the other agent may provide a counterargument, and then the first agent may provide a counter-counterargument. In this way, a depth-first search of the argument graph is undertaken. With the aim of having more natural dialogues, we used a breadth-first approach. So when a user selects arguments from the menu, the system then may attack more than one of the arguments selected. For the argument graph we used in our study, this appeared to work well. However, for larger argument graphs, a breadth-first approach could also be unnatural. This then raises the questions of how to specify a protocol that interleaves depth-first and breadth-first approaches, and of how to undertake studies with participants to evaluate such protocols.

The aim of our dialogues in this paper is to raise the belief in goal arguments. A goal argument may, among other things, incorporate an intention to change behaviour, though we accept that there is a difference between have a intention to do something, and actually doing it. Nonetheless, having an intention to change behaviour is a valuable step to make towards actually changing behaviour. We focus on the beliefs in arguments because belief is an important aspect of the persuasiveness of an argument (see for example \cite{HunterPolberg17ictai}). Furthermore, beliefs can be measured more easily than intentions in crowdsourced surveys. In future work, we would like to investigate to what extent an increased belief in the persuasion translates to actual changes in behaviour. This would be interesting to investigate in healthcare applications, such as persuading participants 
to undertake regular exercise or reduce alcohol intake. 

We also need to investigate new ways of asking for the beliefs in the arguments.
Currently, we do it through a direct question. Unfortunately, this method is vulnerable to participants who misunderstand the instructions, 
select values carelessly or lie on purpose. 
Therefore, we need to create a new, indirect way, of asking for the belief. A possible approach is to develop several simpler, indicative questions (\emph{e.g.} using the \emph{Yes/No} format instead of real values) such that their
answers can be compiled into a numerical value for the belief.  


\section*{Acknowledgements}

This research was funded by EPSRC Project EP/N008294/1 Framework for
Computational Persuasion.


\bibliography{references}
\bibliographystyle{alpha}

\newpage
\section*{Appendix}
\label{section:appendices}

\begin{table}[ht] 
\begin{tabularx}{\textwidth}{p{1.6cm}XXXX} 
\toprule
\multicolumn{1}{p{1.6cm}}{}	& \multicolumn{2}{c}{Advanced System}& \multicolumn{2}{c}{Baseline System} \\
\midrule
{\hbox{\strut Dialogue} 
\hbox{\strut Type}}	&	Normality	&	Significance &	Normality	&	Significance	\\ 
\midrule
All		&	W=0.8482			&	V=4321.5									&	W=0.83112		&		V=2689 	\\  
							& p-value=4.79e-10  	&	p-value=0.01803								&	p-value=2.36e-10	&	p-value=0.83603 \\
Complete 	&	W=0.84615			&	V=3001.5									&	W=0.83796			&	V=1128 	\\  
							& p-value=5.21e-09  	&	p-value=0.024625							&	p-value=1.00e-07	&	p-value=0.88823	\\
Incomplete &	W=0.86144			&	-  &	W=0.81982			&	-	\\    
							& p-value=0.005446  	&												&	p-value=1.2e-05	&	\\  
Linear	    &	W=0.87166			&	V=2129.5									&	W=0.84902			&	V=1272.5 	\\  
							& p-value=3.42e-07  	&	p-value=0.048107							&	p-value=1.14e-07	&	p-value=0.67004	\\
Nonlinear  &	W=0.78658			&	-  &	W=0.78585			&	-	\\   
							& p-value=5.64e-06  	&												&	p-value=6.85e-06	&	\\ 
\multirow{2}{*}{\vtop{\hbox{\strut Keeping} \hbox{\strut Graph}}}  &	W=0.86376			&	V=2941.5									&	W=0.89009			&	V=1910 	\\  
								& p-value=3.96e-08  	&	p-value=0.003792							&	p-value=9.58e-07	&	p-value=0.16493	\\ 
\multirow{2}{*}{\vtop{\hbox{\strut Abolishing} \hbox{\strut Graph}}}	&	W=0.92538			&	t=-1.0592									&	W=0.9582			&	t=-1.63439 	\\
									& p-value=0.060208		&	p-value=0.85018  							&	p-value=0.37979		&	p-value=0.94239	\\
\bottomrule
\end{tabularx}
\caption{Results of analysis of statistical significance of belief changes caused by the APSs on different types of dialogues. Shapiro-Wilk test was used to determine whether the \enquote{before} beliefs were normally distributed. If they were, t-test was used to determine significance of belief changes; otherwise, Wilcoxon signed-rank test was used.  By \enquote{-} we understand that due to the nature of the data, exact p-value could not have been computed, and we make no claims about the significance. }
 \label{tab:statnormdet}
\end{table}

\end{document}